\documentclass{article}
\usepackage{graphicx} 
\usepackage{amsmath, amsthm, amsfonts, amssymb}
\usepackage{a4wide}
\usepackage{fontenc}
\usepackage{comment}
\usepackage{tikz,pgfplots}
\newtheorem{lemma}{Lemma}
\newtheorem{proposition}{Proposition}
\newtheorem{theorem}{Theorem}
\newtheorem{corollary}{Corollary}
\newtheorem{remark}{Remark}
\pgfplotsset{compat=1.18} 
\usepackage{algorithm,multicol}
\usepackage{algpseudocode}
\usepackage[normalem]{ulem}
\usepackage{mathtools}
\mathtoolsset{showonlyrefs}

\usepackage{caption}
\usepackage{subcaption}
\usepackage{verbatim}
\usepackage{hyperref}

\usepackage{todonotes}

\def\E{\mathbb{E}}

\def\V{\mathbb{V}}

\newcommand{\norm}[1]{\left\lVert#1\right\rVert}

\newcommand{\Z}{\mathbb{Z}}

\newcommand{\R}{\mathbb{R}}
\newcommand{\N}{\mathbb{N}}

\newcommand{\calG}{\mathcal{G}}
\newcommand{\calO}{\mathcal{O}}

\newcommand{\calX}{\mathcal{X}}

\DeclareMathOperator{\ent}{Ent}

\newcommand{\fraku}{\mathfrak{u}}

\newcommand{\JMN}{\widehat{J}_{M,N}}

\newcommand{\rd}{\mathrm{d}}

\usepackage{xcolor}
\definecolor{darkblue}{RGB}{0,60,180} 
\definecolor{darkgreen}{RGB}{0,130,70}
\definecolor{darkorange}{RGB}{180,60,0}

\newtheorem{assumptionA}{Assumption}
\newcommand{\keywords}[1]{\par\medskip\noindent\textbf{Keywords:} #1}
\newcommand{\subclass}[1]{\par\noindent\textbf{Mathematics Subject Classification:}
  {\def\and{\unskip,\ }#1}}

\title{Approximation of differential entropy in Bayesian optimal experimental design}
\author{Chuntao Chen, Tapio Helin, Nuutti Hyvönen, Yuya Suzuki}
\date\today

\begin{document}

\maketitle

\begin{abstract}
    Bayesian optimal experimental design provides a principled framework for selecting experimental  configurations that maximize obtained information. In this work,  we consider
settings where the likelihood entropy is independent of the design or can be
evaluated explicitly, which applies,~e.g.~to additive noise models with design-independent or explicitly computable noise
entropy.
    This reduces the problem to maximum entropy estimation,  where the nested integral defining the expected information gain 
collapses into a single integral. 

Our study is motivated by large-scale inference problems, such as inverse problems, where the computational cost is dominated by expensive likelihood evaluations.  Under additive Gaussian noise model, we propose a computational approach in which the evidence density is approximated by a Monte Carlo or quasi-Monte Carlo surrogate, while the differential entropy is evaluated using standard methods without additional likelihood evaluations. We prove that this strategy achieves convergence rates that are comparable to, or better than, state-of-the-art methods for full expected information gain estimation, particularly when the cost of entropy evaluation is negligible. Our approach relies only on mild smoothness of the forward map and avoids stronger technical assumptions required in earlier work. We also present numerical experiments, which confirm  some of our theoretical findings.

\keywords{Bayesian optimal experimental design, differential entropy, evidence distribution, Gaussian mixture model, Monte Carlo, quasi-Monte Carlo.}

\subclass{62K05\and 62F15\and 65D30\and 65D32}
\end{abstract}

\section{Introduction}
\label{sec:intro}

Despite the rapid growth of computational resources in science and engineering, observational data remain constrained due to financial or physical limitations. Illustrative examples include medical imaging, where excessive radiation exposure of patients must be avoided, and seismic imaging, where the high cost of additional measurements is prohibitive. In such contexts, optimizing experiments becomes crucial to making the most effective use of limited resources. Bayesian optimal experimental design (OED) provides a principled approach to formalizing the process of selecting experiments that best address uncertainty and improve the accuracy of model predictions.

To formally describe the Bayesian OED paradigm, we begin by establishing some notation. Let $x$ denote the quantity of interest that attains values in a separable Banach space ${\mathcal X}$. Our initial beliefs about $x$, i.e.~the prior information, are captured by a probability measure $\mu$ defined on $({\mathcal X}, {\mathcal B}({\mathcal X}))$. The probability of observing a data vector $y\in {\mathcal Y} = \R^{d}$, given the unknown $x$, is modeled by the likelihood density $\pi(y \mid x; \xi)$ corresponding to a regular conditional probability of $y$ given $x$ with $\xi \in {\mathcal D}$ representing the design variable identifying the experiment. Notice that the measurement domain ${\mathcal Y}$ is assumed to be independent of the design.
The prior and the likelihood together compose the Bayesian joint distribution, which, through Bayes' theorem, gives rise to the posterior distribution $X \mid Y =y$, denoted here by $\mu^y$.

The goal of OED is to find the design $\xi$ that maximizes the expected utility or selection criteria over the Bayesian joint distribution, i.e., we maximize
\begin{equation*}
	U(\xi) = \E\, u(X, Y; \xi)
\end{equation*}
with respect to $\xi \in {\mathcal D}$. In practical experimental design, Bayesian OED has been constrained by prohibitive computational costs \cite{ryan16}. However, recent advances in computational techniques and resources have allowed OED to gain traction in tackling large-scale problems. The aim of this paper is to introduce scalable computational methods that address and overcome the computational challenges associated with a class of Bayesian OED tasks.

Different utilities $u$ have been studied in literature (see,~e.g.,~\cite{Pukelsheim06}), and an effective choice depends on the objectives of the experiment. Here, we focus on the \emph{expected information gain} (EIG), where the utility $u(X,Y;\xi)$ is given by the Kullback--Leibler (KL) divergence between the posterior distribution and the prior. The EIG is formulated as
\begin{equation}
\label{eq:EIG}
U(\xi) := \E \, D_{\rm KL}(\mu^Y, \mu) = \iint_{\R^d \times \mathcal{X}} \log \! \Big(\frac{\rd \mu^y}{\rd \mu}(x)\Big) \nu(\rd x, \rd y; \xi),
\end{equation}
where $\nu(\,\cdot\,, \,\cdot\,; \xi)$ stands for the joint Bayesian distribution. 

The immediate computational challenge in \eqref{eq:EIG} stems from the intractable integrand that requires nested estimations.
Often, one rephrases the integrand via Bayes' formula $(\rd \mu^y / \rd \mu)(x;\xi) = \pi(y \mid x; \xi) / \pi(y;\xi)$ given the evidence density $\pi(y; \xi) = \int \pi(y \mid x ; \xi) \mu(\rd x)$. This yields
\begin{align}
    \label{eq:U_disintegrated}
    U(\xi) &= \iint_{\R^{d} \times \mathcal{X}} \log \frac{\pi(y \mid x; \xi)}{\pi (y; \xi)} \, \nu(\rd x,\rd y;\xi)\nonumber\\[1mm]
    & =  - \int_{\R^{d}} \log \left(\pi\left(y;\xi\right)\right) \pi(y;\xi) \, \rd y + \iint_{\R^{d} \times \mathcal{X}} \log \left(\pi\left(y \mid x;\xi\right)\right)\pi(y \mid x;\xi) \, \rd y \, \mu(\rd x)
    \nonumber\\[1mm]
    & =  \ent(\pi(\, \cdot \,; \xi)) - \E^\mu \ent(\pi(\, \cdot \, \mid X; \xi)),
\end{align}
where the two terms on the right-hand side are the {\em evidence entropy} and the $\mu${\em-averaged likelihood entropy} in accordance with the definition
\begin{equation}
    \label{eq:def_diff_entropy}
    \ent(\rho) = -\int_{\R^{d}} \rho(y) \log(\rho(y)) \, \rd y
\end{equation}
of the differential entropy for a probability density $\rho$. The decomposition in \eqref{eq:U_disintegrated} reveals that if $\mu$-averaged likelihood entropy is independent of the design $\xi$, then maximizing the EIG is equivalent to maximizing
the evidence entropy. More generally, if this likelihood-entropy term is available
in closed form as a function of $\xi$, then the EIG can be evaluated by estimating
only the evidence entropy and adding the known correction. This observation is
not tied to Gaussian likelihoods considered in what follows, but it applies whenever the entropies in \eqref{eq:U_disintegrated} are well defined and the likelihood entropy is either
design-independent or explicitly computable.

However, the evidence $\pi(\, \cdot \, ;\xi)$ in \eqref{eq:U_disintegrated} remains intractable, necessitating a separate estimation. In this case, it is well-known that nested Monte Carlo (MC) estimation of EIG converges at rate ${\mathcal O}(M^{-\frac 13})$, where $M$ is the number of likelihood evaluations \cite{ryan16,huan2024optimal}. In the context of expensive computational models such as \eqref{eq:compmodel} below, such slow convergence is computationally prohibitive. It has been the focus of ongoing research for the last decade or so to design computationally effective approximative schemes that accelerate this convergence (on recent work, see,~e.g.,~\cite{bartuska2023double,bartuska2024multilevel,kaarnioja2024quasimontecarlobayesiandesign}). 

Our work focuses on large-scale inference problems in which the unknown influences the likelihood distribution solely through its mean. This is the case,~e.g.,~when the observation is corrupted by additive noise. In particular, we are motivated by inverse problems \cite{engl,stuart-acta-numerica}, which constitute a rich class of problems, where the likelihood can typically be given in closed form but is expensive to evaluate due to a complex underlying mathematical model connecting the unknown and the data, such as a partial differential equation (PDE). In many inverse problems, the additive (Gaussian) noise model stands for a convenient proxy for the observational uncertainty as a more detailed noise structure is typically unavailable. In what follows, we assume that the likelihood is induced by the computational model
\begin{equation}
	\label{eq:compmodel}
    y = {\mathcal G}(x; \xi) + \epsilon(\xi),
\end{equation}
where $\calG:\mathcal{X}\to\mathcal{Y}$ is the forward map simulating the experiment and the noise vector $\epsilon$, with density $\eta(\,\cdot\,; \xi)$, is independent of $x$ but may depend on $\xi$. 
Since the differential entropy in \eqref{eq:def_diff_entropy} is translation-independent with respect to $\rho$, the additive noise in \eqref{eq:compmodel} implies that 
\begin{equation*}
    \E^\mu \ent(\pi(\, \cdot \,  \mid X; \xi)) = \ent(\eta(\, \cdot \, ;\xi)).
\end{equation*}
For example, if $\epsilon$ is Gaussian with covariance matrix $\Gamma(\xi)$, as is often assumed in practical applications, the second term in \eqref{eq:U_disintegrated} satisfies~\cite{Shannon48}
\begin{equation}
    \ent(\eta(\, \cdot \, ;\xi)) = \frac{1}{2}d(1 + \log 2 \pi) + \frac 12 \log \det \Gamma(\xi).
\end{equation}
Consequently, the evaluation of the expected utility in \eqref{eq:U_disintegrated} is simplified and boils down to estimating the differential entropy
\begin{equation}
\label{eq:J}
J(\xi) :=  \ent(\pi(\, \cdot \,; \xi)) = -\int_{\R^{d}} \pi(y; \xi) \log(\pi(y; \xi)) \, \rd y
\end{equation} 
of the evidence distribution. This well-known simplification is referred to as maximum entropy sampling \cite{shewry1987maximum,sebastiani2000maximum}, but we note that the same term has also been used to describe combinatorial design problems arising as special cases, see \cite{fampa2022maximum}.

Interestingly, the optimization of the evidence entropy $J$ has received relatively limited attention in experimental design literature. Recently, Foster et al.~\cite{foster2019variational} analysed variational inference for Bayesian OED in a number of settings, including variational approximation of the evidence density in evaluations of the full EIG utility. In particular, the authors provide a convergence result \cite[Theorem~1]{foster2019variational} balancing the steps used to refine the variational approximation in concert with the number of samples used for the MC estimator. While the result relies on potentially involved technical assumptions, the key requirement for convergence of the variational approximation is that the true evidence lies within the variational family. A more closely related study is presented in \cite{GGNP2019}, focusing on the convergence of empirical distributions under smoothing. This aligns directly with our objectives, and the connection, particularly with our results in Section~\ref{sec:GMM}, is discussed in detail below.

In this work, we propose and analyse direct approximation of $J(\xi)$ in two sub-tasks: first, {\em density estimation of the evidence} based on prior samples of the unknown $x$, push-forwarded through an approximate version of $\calG$ that can be obtained,~e.g.,~via discretization of the underlying PDE; second, {\em estimation of the differential entropy of the approximated evidence}. In particular, the first sub-task is more relevant for us as the approximated evidence is fast to evaluate and its differential entropy can be estimated by fast off-the-shelf kernel density estimation software packages (see e.g. \cite{kraskov_estimating_2004,duong2007ks}).

\subsection{Our contribution}

Our core contribution is introducing estimators for $J$ in \eqref{eq:J} under a Gaussian likelihood model (even though the above reduction from EIG to evidence-entropy estimation is valid more generally), together with rigorous convergence guarantees that improve previous literature as detailed below.
The error analysis is based on the following three components. First, we assume access to an efficient surrogate model ${\mathcal G}_K$ that approximates ${\mathcal G}$ as $K$ increases. Second, by pushing the prior samples or cubature nodes forward through ${\mathcal G}_K$ and incorporating the likelihood, we define a surrogate evidence $\pi_M^K(\, \cdot \, ;\xi)$, where $M$ denotes the size of the unsupervised training data. Third, we approximate the differential entropy of the evidence surrogate using a MC estimator, that is, we analyze the properties of the estimator
\begin{equation}\label{eq:def-estimator}
    \JMN^K(\xi) = -\frac 1N \sum_{n=1}^N \log \big(\pi_M^K\left(Y_n;\xi\right)\big), 
\end{equation}
where $Y_n \sim \pi_M^K$ i.i.d. The main benefit of this approach is that the convergence rate depends on the dimension of the input domain ${\mathcal X}$ only through the approximation rate of ${\mathcal G}_K$. This statement holds with no additional condition if the training data are samples from the prior; on the other hand, for quasi-Monte Carlo (QMC) cubature nodes considered in Section~\ref{sec:QMC-GMM}, the convergence rates themselves are dimension-independent, but the ``hidden constants'' may depend on the dimension of ${\mathcal X}$.

We note that our motivation arises from scenarios, where evaluating $\pi_M^K$ is considerably inexpensive in comparison to mapping with the forward operator $\calG$ and, consequently, the generation of the training points. Therefore, the choice of an MC estimator for the differential entropy of $\pi_M^K$ is not central to our work, and could, from the standpoint of analysis, be directly replaced by another method such as sparse grids \cite{BG2004} or QMC methods \cite{DKS2013,N1992} with theoretically established convergence rates for general tasks, or by other techniques specifically developed for differential entropy approximation \cite{Nonparametric_entropy_estimation_overview_1997,nemenman2002entropyinferencerevisited,Shalev2024}. In fact, we experiment on such an idea in our numerical experiments (see Remark~\ref{remark:cubature}) to better isolate the effect of the sample complexity $M$ in the convergence.

In this paper, we formulate the surrogate evidence $\pi_M^K$ as a Gaussian mixture model (GMM) with $M$ components that follow, respectively, normal distributions $\mathcal{N}({\mathcal G}_K(x_m,\xi),\Gamma)$, where $\Gamma$ is predetermined and the nodes $x_m$ are either independently generated by the prior or are designed by some specific cubature rules. In this setting our contributions include: 
\begin{itemize}
    \item We provide a bias-variance decomposition for the total error of an MC estimator of the differential entropy for a general evidence surrogate $\pi_M^K$ in Theorem~\ref{thm:theorem1}. In particular, we observe that the KL divergence provides a natural context for the evidence approximation.
    \item In Theorem~\ref{thm:MC-GMM-estimator}, we characterize the total sampling error of an evidence approximation based on random i.i.d. samples from a sub-Gaussian prior and a Lipschitz continuous forward mapping. We show the error is controlled by the modeling error ${\mathcal G}-{\mathcal G}_K$ and the sample sizes $M$ and $N$. The root mean-squared error (RMSE) of our differential entropy estimator,~i.e.,
    \[
    \sqrt{\E\big[|\JMN^K(\xi)-J(\xi)|^2\big]},
    \]
    where the expectation is taken over the randomization in the construction of the estimator, converges as $\calO(\delta_K + N^{-1/2}+M^{-1/2})$, where $\delta_K$ is the root mean-squared modeling error averaged over the prior.
    \item Assuming additional regularity of the forward mapping ${\mathcal G}_K$ and employing an evidence surrogate constructed from an ensemble of points $\{{\mathcal G}_K (x_j)\}_j$ specified by a randomized first-order QMC-based cubature, we obtain the convergence rate $\calO(\delta_K + N^{-1/2} + M^{-1+ \gamma})$ for any $\gamma > 0$. This represents accelerated convergence with respect to the number of forward map evaluations compared to Theorem~\ref{thm:theorem1}. The result is derived for the uniform prior in Theorem~\ref{thm:QMC-uniform-result} and for a Gaussian prior in Theorem~\ref{prop:E1-QMC-Gauss}. In addition, Remark~\ref{rem:ho-qmc} quantifies how higher-order QMC cubatures lead to higher convergence rates in $M$ for the uniform prior under more restrictive regularity assumptions on ${\mathcal G}_K$. 
    \item In Section \ref{sec:numerics}, we demonstrate that the numerical convergence rate in $M$ aligns with our theory for two applications: a linear deconvolution problem, which enables us to compare with the ground truth due to availability of a closed-form solution, and a nonlinear Darcy flow problem. (We do not systematically consider convergence in $N$ or $K$ in the numerical tests.) 
\end{itemize}
We note that the employed mixture model can be interpreted as a kernel density estimation method \cite{silverman2018density}, thereby inviting the consideration of alternative kernel functions. Here, the choice of GMM framework allows us to leverage theoretical convergence results from recent work \cite{GGNP2019}, which is crucial for the analysis that follows. The authors in \cite{GGNP2019} derive a result much aligned with our Theorem \ref{thm:MC-GMM-estimator} with the same convergence rate with respect to the number of training points. In addition, \cite{GGNP2019} provides a minimax optimality result for a class of sub-Gaussian evidences. Our theoretical results include the effect of discretization and go beyond to analyze the QMC-based evidence surrogate to demonstrate that accelerated schemes are possible.

Closely connected to these findings, we note that the minimax optimal rates of kernel density estimation are well-known for a wide class of settings. However, they typically involve a dimension-dependent rate which degrades as the problem dimension increases. In short, typical minimax rate for learning a density on $\R^d$ behaves as $M^{-1/d}$ for large $d$. The key finding of \cite{GGNP2019} is that for smoothened densities, such as the evidence density here, one can recover dimension-independent convergence rates when the signal-to-noise ratio is bounded from above. Here, we establish in Remark \ref{rem:bounded_signal_to_noise} that our assumptions on the evidence structure indeed imply a similar constraint.

To compare obtained convergence rates with the state of the art, we mention that in \cite{kaarnioja2024quasimontecarlobayesiandesign} the authors employ randomized QMC methods for estimating EIG either (i) by a tensor product of two cubature rules over $x$ and $y$ achieving the error convergence $\calO(M^{-1/2})$; or (ii) by Smolyak construction of combining two cubature rules achieving $\calO(M^{-1})$, which is inspired by sparse tensor product quadrature rules \cite{GGO2022}. However, the authors assume a bounded finite-dimensional input domain and require arbitrarily high smoothness of the forward mapping, more precisely, the norm of $\partial^\nu {\mathcal G}$ is bounded for an arbitrary multi-index $\nu$ with a specific growth asymptotics as $|\nu|$ increases. Moreover, the effect of discretization of ${\mathcal G}$ is not considered. 
A more recent study \cite{bartuska2024multilevel} introduces a multi-level double-loop QMC estimator taking into account the discretization and achieves an error tolerance $\text{TOL}$ at a computational cost of nearly $\mathcal O(\text{TOL}^{-1-\omega/\eta})$ operations, where $\omega$ and $\eta$ characterize, respectively, the cost of evaluating $\mathcal G_K$ and its approximation rate. 
The convergence rates demonstrated in \cite{kaarnioja2024quasimontecarlobayesiandesign,bartuska2024multilevel} involve an additional logarithmic term due to the truncation of the outer integral (see \cite[Remark~10]{bartuska2024multilevel}). 
%The logarithmic 
Such a term does not appear in our rates as the convergence analysis for the differential entropy estimation is not coupled with the major computational cost of evaluating ${\mathcal G}_K$. That being said, it does involve additional computational overhead which will be analyzed in future work.

To facilitate comparison with related work, Table~\ref{tab:cost-comparison}
summarizes simplified asymptotic cost required to achieve an RMSE tolerance
$\mathrm{TOL}$. 
We assume the forward-model approximation error satisfies $\delta_K=\mathcal{O}(K^{-\eta})$
and that one evaluation of $\mathcal{G}_K$ costs $\mathcal{O}(K^\omega)$,
so that a forward solve at the accuracy $\delta_K=\mathcal{O}(\mathrm{TOL})$ costs $\mathcal{O}(\mathrm{TOL}^{-\omega/\eta})$.
The table is intended as a comparison of the dominant cost associated with forward-model evaluations. Logarithmic factors and arbitrarily small losses in QMC convergence exponents are suppressed unless explicitly indicated. The nominal convergence exponents alone do not provide a complete ordering of the methods: the associated regularity assumptions, dimension dependence of QMC constants, discretization treatment etc.~differ substantially between the approaches, as briefly analyzed in what follows.

\begin{table}[h!]
\centering
\small
\caption{
Simplified asymptotic forward-solve cost for achieving an RMSE tolerance $\mathrm{TOL}$. Here $\delta_K=O(K^{-\eta})$ and the cost of evaluating $\mathcal{G}_K$ is $O(K^\omega)$. For the QMC results of this work, $\varepsilon>0$ denotes an arbitrarily small loss in the convergence exponent. See the text for information on dimension-dependent constants, postprocessing costs, and method-specific assumptions.
}
\label{tab:cost-comparison}
\begin{tabular}{@{}lll@{}}
\hline
Method & Regime & Forward-solve cost \\
\hline
Nested MC & standard nested EIG & 
$\mathcal{O}(\mathrm{TOL}^{-3-\omega/\eta})$ \\

Foster et al.~\cite{foster2019variational} & variational / MC & 
$\mathcal{O}(\mathrm{TOL}^{-2-\omega/\eta})$ \\

Kaarnioja--Schillings~\cite{kaarnioja2024quasimontecarlobayesiandesign} & sparse tensor QMC, nested &
$\mathcal{O}(\mathrm{TOL}^{-1-\omega/\eta})$ \\

Bartuska et al.~\cite{bartuska2024multilevel} & multilevel randomized QMC &
$\mathcal{O}(\mathrm{TOL}^{-1-\omega/\eta})$ \\

This work, MC & GMM surrogate &
$\mathcal{O}(\mathrm{TOL}^{-2-\omega/\eta})$ \\

This work, QMC lattice & GMM surrogate &
$\mathcal{O}(\mathrm{TOL}^{-1-\omega/\eta-\varepsilon})$ \\

This work, HO-QMC order $\alpha$ & GMM surrogate &
$\mathcal{O}(\mathrm{TOL}^{-1/\alpha-\omega/\eta-\varepsilon})$ \\
\hline
\end{tabular}
\end{table}

The last three rows are obtained by balancing the $K$- and $M$-dependent error contributions and neglecting the $N$-dependent cost of the final entropy evaluation; see Theorems~\ref{thm:MC-GMM-estimator} and~\ref{thm:QMC-uniform-result} and Remark~\ref{rem:ho-qmc}. Accordingly, the costs reported for the methods proposed in this work should be interpreted as forward-solve complexity estimates rather than complete computational complexity estimates.

For the QMC and higher-order QMC results of this work, the displayed costs account for the convergence exponents but {\em do not} account for possible dependence of the QMC error constants on the approximation dimension $K$. Consequently, these rows should be regarded as formal asymptotic forward-solve estimates unless the relevant constants can be controlled uniformly as $K$ increases.
This dimension dependence is generally unavoidable under our assumptions. 
In contrast, it should be emphasized that Kaarnioja and Schillings~\cite{kaarnioja2024quasimontecarlobayesiandesign} obtain dimension-robust QMC estimates with respect to the parameter dimension, which comes at the price of substantially stronger assumptions on the parametric forward map. Note, however, that under their stronger weighted regularity and summability assumptions, the QMC construction driving our evidence surrogate can inherit their dimension-robust inner-integration bound; see Remark~\ref{rem:qmc-dependence}. The constants in~\cite{kaarnioja2024quasimontecarlobayesiandesign} may still depend on other quantities such as the data dimension and the data-domain truncation, and as~\cite{kaarnioja2024quasimontecarlobayesiandesign} does not include forward-model discretization, the factor $\mathrm{TOL}^{-\omega/\eta}$ in Table~\ref{tab:cost-comparison} is due to the forward-cost model introduced above.

The multilevel analysis in Bartuska et al.~\cite{bartuska2024multilevel} incorporates the level-dependent cost of the discretized forward model directly into the work estimate. The factor associated with forward discretization is therefore not appended afterwards using our cost model. The displayed rate is a simplified representation of their multilevel randomized-QMC complexity: logarithmic factors, arbitrarily small QMC losses, and possible dependence of constants on the dimensions of variables are suppressed. On the other hand, for Foster et al.~\cite{foster2019variational}, the $\mathcal{O}(\mathrm{TOL}^{-2})$ contribution assumes that the variational approximation error can be reduced to the prescribed tolerance. The additional factor $\mathrm{TOL}^{-\omega/\eta}$ in Table~\ref{tab:cost-comparison} results from applying the common forward-discretization cost model.

To summarize, our results demonstrate that when the differential entropy of the likelihood is independent of the design or can be explicitly evaluated, it is advantageous to exploit this structure and employ maximum entropy estimation. In the regime where forward solves dominate, the resulting nominal complexity rates are comparable to state-of-the-art approaches, while requiring relatively mild smoothness assumptions on the forward map and avoiding some of the stronger technical assumptions employed in existing QMC analyses.

All our results hold pointwise in the design $\xi$ and, under suitable assumptions, extend to uniform validity over the design domain. To streamline the exposition, we omit the explicit dependence on the design in the notation throughout the paper.

\subsection{Other related work}

Bayesian experimental design has a rich history with extensive literature. We refer to \cite{huan2024optimal,rainforth2024modern,ryan16} as recent general overviews. Moreover, a broad discussion on the various different utilities is given in \cite{chaloner1995bayesian}. See also recent work on Wasserstein distance–based utilities in \cite{helin2025wasserstein}.  In our work, we focus specifically on the expected information gain criterion, a concept often attributed to Lindley's foundational contribution \cite{lindley1956measure}.

Our results are closely related to the recent work by Foster et al.~\cite{foster2019variational,foster2020unified}, who explored variational approximations to compute nested integrations. Particularly relevant to our approach, their investigation on variational approximation of the evidence demonstrated a convergence rate of $\mathcal{O}(M^{-1/2} + N^{-1/2})$ in terms of RMSE, where the variational evidence approximation occurs at order $\mathcal{O}(M^{-1/2})$ and MC error occurs at order $\mathcal{O}(N^{-1/2})$ \cite{foster2019variational}; see also Table~\ref{tab:cost-comparison}.

Several key distinctions separate our work from these previous efforts. First, while Foster et al.~assume representation of the target distribution in a finite-dimensional latent space, the GMM approach can approximate a non-parametric family of evidence distributions. Second, our method is a direct approximation scheme requiring no additional computational effort beyond sampling and mapping the prior cubature. Third, we demonstrate that QMC cubatures, which leverage mapping properties in the mathematical model, can achieve even further acceleration in the convergence rates.

Inverse problems constitute a class of high-dimensional inference challenges where complex mathematical models such as PDEs connect unknown parameters to observable data. The need for scalability across various discretization levels in inverse problems has catalyzed research extending traditional Bayesian experimental design criteria to infinite-dimensional settings \cite{alexanderian2021optimal}.
Moreover, the standard nested MC estimators typically require a prohibitive computational effort and various approaches have been proposed to reduce the computational cost. We mention the avenues of research involving Laplace or Gaussian approximation (see,~e.g.,~\cite{long2013fast,BECK2018523,bartuska2025laplace,wu2023_laplace,helin2022edge}), neural network based surrogates \cite{wu2023large,kleinegesse2020bayesian,orozco2024probabilistic,go2025sequential} and multi-level MC \cite{goda2022,goda2020multilevel,beck2020multilevel}. In addition, QMC methods have been employed in \cite{kaarnioja2024quasimontecarlobayesiandesign,bartuska2024multilevel}. Building on these ideas, direct estimation of a gradient of the expected utility has been considered,~e.g.~in \cite{goda2022}.

Gaussian mixture models have been investigated for entropy estimation applications in multiple studies \cite{GGNP2019,kolchinsky2017}. The convergence rate of entropy estimation with respect to the KL divergence has been established, with applications primarily focused on neural networks rather than Bayesian experimental design \cite{GGNP2019}. For a family of estimators based on Gaussian mixture models, both upper and lower bounds have been derived using the distance function between mixture components \cite{kolchinsky2017}.

\subsection{Outline}
This paper is organized as follows. Section \ref{sec:bias-var_decomp} decomposes the mean squared error between the differential entropy $J$ from \eqref{eq:J} and the estimator $\JMN^K$ from \eqref{eq:def-estimator} into two parts, the bias and the variance, without specifying the technique for forming the surrogate density $\pi^K_M$. Section \ref{sec:GMM} derives the convergence rate for the MC-based GMM variant of \eqref{eq:def-estimator} in terms of the sample sizes $N$ and $M$ and the expected error between $\mathcal{G}$ and its surrogate $\mathcal{G}_K$, under certain assumptions on the prior, $\mathcal{G}$ and $\mathcal{G}_K$. The convergence rate is further accelerated in Section~\ref{sec:QMC-GMM} for uniform and Gaussian priors using QMC points as training data in $\mathcal{X}$. Section~\ref{sec:numerics} presents numerical examples that demonstrate the established convergence rates in $M$ for our approach, using both MC and QMC to build the GMM. Lastly, Section \ref{sec:conclusion} presents the concluding remarks and discusses future work.

\section{MC estimator and bias-variance decomposition}\label{sec:bias-var_decomp}

In this section, we assume to be given an approximative forward operator ${\mathcal G}_K$ that is practically implementable and gives rise to an evidence distribution $\pi^K$ under the likelihood model induced by \eqref{eq:compmodel}. Moreover, we assume that $\pi^K$ can be approximated by $\pi_M^K$ that is constructed using an unsupervised training data set $\{x_m\}_{m=1}^M \subset {\mathcal X}$ and ${\mathcal G}_K$. At this stage, we do not specify the particular approximation scheme for forming $\pi_M^K$ but treat it as a general surrogate that converges to the true evidence $\pi$ as both $M$ and $K$ increase. Moreover, note carefully that throughout this section $\pi^K_M$ is treated as a fixed probability density, whereas in the following sections it will become random due to the randomization of the set $\{x_m\}$.

Let us consider the MC estimator $\JMN^K$ defined in \eqref{eq:def-estimator} with the help of $\pi_M^K$ and make some immediate observations about its first and second order statistics. Recall that $J$ is the differential entropy of the evidence defined by \eqref{eq:J},~i.e.,~the quantity we aim to approximate throughout this work. Here and in what follows, we employ the notation
\[
\E^{\otimes \pi_M^K} [\JMN^K] :=
    \int_{\R^{d}} \cdots \int_{\R^{d}} \JMN^K( y_1,\dots,y_N ) \, \pi^K_M(y_1) \, \rd y_1\cdots \, \pi^K_M(y_N) \, \rd y_N  .
\]
The following bias-variance decomposition does not depend on the Gaussian likelihood assumption. However, the quantitative bound in Theorem~\ref{thm:theorem1}, and all subsequent convergence results, are only established for the additive Gaussian noise model introduced below.

\begin{proposition}
\label{prop:total_error}
For any $M,N>0$, we have
\begin{equation}
    \label{eq:expec_piMK}
    \E^{\otimes \pi_M^K} [\JMN^K] =
    \ent(\pi_M^K),
\end{equation}
and the mean squared error is given by
\begin{equation}
    \label{eq:total_error_decomposition}
	\E^{\otimes \pi_M^K}\big| J  - \JMN^K\big|^2 = \big(\ent(\pi) - \ent(\pi_M^K)\big)^2 + \frac 1N  \V_{\pi_M^K} \big(\log(\pi_M^K(Y)) \big).
\end{equation}
\end{proposition}

\begin{proof}
The identity \eqref{eq:expec_piMK} follows directly from \eqref{eq:def-estimator} and the definition of differential entropy in \eqref{eq:def_diff_entropy}.
Moreover, we have
\begin{align*}
    \E^{\otimes \pi_M^K} (\JMN^K)^2 & =  
    \frac 1{N^2} \E^{\otimes \pi_M^K} \bigg(\sum_{n=1}^N \log^2 (\pi_M^K(Y_n)) + \sum_{k\neq \ell} \log (\pi_M^K(Y_k))\log (\pi_M^K(Y_\ell))\bigg) \\
    & =  \frac 1{N} \E^{\pi_M^K} \log^2 (\pi_M^K(Y)) + \frac{N-1}{N} \big(\E^{\pi_M^K} \log (\pi_M^K(Y))\big)^2 \\[1mm]
    & =  \ent(\pi_M^K)^2 + \frac 1N \E^{\pi_M^K} \Big(\log(\pi_M^K(Y))- \E^{\pi_M^K} \log(\pi_M^K(Y))\Big)^2,
\end{align*}
which yields the assertion about the mean squared error.
\end{proof}

Let us next formulate an auxiliary upper bound for the entropy difference in identity \eqref{eq:total_error_decomposition}, forming the basis for the forthcoming analysis. To that end, define the $\chi^2$-distance between $\mu_1$ and $\mu_2$ as
\begin{equation*}
    \chi^2(\mu_1, \mu_2) = \E^{\mu_2}\bigg(\frac{\rd \mu_1}{\rd \mu_2}(Z)-1\bigg)^2
\end{equation*}
whenever $\mu_1 \ll \mu_2$,~i.e.,~$\mu_1$ is absolutely continuous with respect to $\mu_2$. If $\mu_1$ and $\mu_2$ are defined on $\R^l$ with densities $\pi_1$ and $\pi_2$, respectively, we adopt the convention $\chi^2(\pi_1, \pi_2) = \chi^2(\mu_1, \mu_2)$. By the concavity of the logarithm and Jensen's inequality,
\begin{align}
    \label{eq:KL_bounded_by_chi2}
    D_{\rm KL}(\pi_1, \pi_2) &\leq \log \! \bigg( \int_{\R^l} \frac{\pi_1(z)^2}{\pi_2(z)} \, \rd z \bigg) \nonumber \\&=  \log \! \bigg(\int_{\R^l} \bigg(\frac{\pi_1(z)^2}{\pi_2(z)^2} -  2 \frac{\pi_1(z)}{\pi_2(z)} +1  \bigg) \pi_2(z) \, \rd z + 1 \bigg) \nonumber \\[1mm] 
    &= \log\big( 1 + \chi^2 (\pi_1, \pi_2) \big) \leq \min \! \big\{\chi^2 (\pi_1, \pi_2), \sqrt{\chi^2 (\pi_1, \pi_2)} \big\},
\end{align}
where we applied $\log(1+t)\leq \min\{t, \sqrt{t}\}$ for $t\geq 0$. In consequence, the KL divergence is bounded by the $\chi^2$-distance, and also by its square root, as is well-known.

\begin{lemma}\label{lem:entropy-KL-bound}
Suppose $\mu_1$ and $\mu_2$, with $\mu_1 \ll \mu_2$, are probability measures on $\R^l$ with densities $\pi_1$ and $\pi_2$, respectively. The following bounds hold:
\begin{equation}
    \left|\ent(\pi_1) - \ent(\pi_2)\right|
    \leq \left(\sqrt{\E^{\pi_2} \log^2 \pi_2(Z)} + 1\right) \, \sqrt{\chi^2(\pi_1, \pi_2)}\label{eq:diff-entr-chi2-ineq}
\end{equation}
and
\begin{equation}
    \left|\ent(\pi_1) - \ent(\pi_2)\right|
    \leq \sqrt{2(\E^{\pi_1} + \E^{\pi_2}) \left( \log^2 \pi_2(Z)\right)} \, \sqrt{D_{\rm KL}(\pi_1 , \pi_2)} + D_{\rm KL}(\pi_1, \pi_2) .\label{eq:diff-entr-KL-alt-ineq}    
\end{equation}
\end{lemma}

\begin{proof}
We decompose the difference into two terms
\begin{equation}
    \label{eq:general_ent_diff}
    \ent(\pi_2) - \ent(\pi_1) = \int_{\R^l} \log (\pi_2(z)) \, (\pi_1(z)-\pi_2(z)) \, \rd z
    + D_{\rm KL}(\pi_1, \pi_2).
\end{equation}
Applying the Cauchy--Schwarz inequality to the first term yields
\begin{align*}
    &\left|\int_{\R^l} \log (\pi_2(z)) \, (\pi_2(z)-\pi_1(z)) \, \rd z\right| 
    \\ & \leq  \left(\int_{\R^l} \log^2 (\pi_2(z)) \, \pi_2(z) \, \rd z\right)^{1/2} \!
    \left(\int_{\R^l} \frac{\left(\pi_1(z)-\pi_2(z)\right)^2}{\pi_2(z)} \, \rd z\right)^{1/2} \\
    & =  \sqrt{\E^{\pi_2} \log^2 \pi_2 (Z) } \, \sqrt{\chi^2(\pi_1 , \pi_2)}.
\end{align*}
Combined with the square root bound in \eqref{eq:KL_bounded_by_chi2} and \eqref{eq:general_ent_diff}, this proves  \eqref{eq:diff-entr-chi2-ineq}.

The alternative bound in \eqref{eq:diff-entr-KL-alt-ineq} follows via a simple modification of  the argument:
\begin{align*}
    &\left|\int_{\R^l} \log (\pi_2(z)) \, (\pi_2(z)-\pi_1(z)) \, \rd z\right| 
    \\& \leq  2 \left(\int_{\R^l} \log^2 (\pi_2(z))  \Big(\sqrt{\pi_2(z)}+\sqrt{\pi_1(z)}\Big)^2 \rd z\right)^{1/2} \! \!
    D_{\rm Hell}(\pi_1, \pi_2) \\[1mm]
    & \leq  \sqrt{2} \sqrt{(\E^{\pi_2} + \E^{\pi_1}) \big( \log^2 \pi_2(Z)\big)} \, \sqrt{D_{\rm KL}(\pi_1 , \pi_2)},
\end{align*}
as the Hellinger distance 
\[D_{\rm Hell}(\pi_1, \pi_2):=\bigg(\frac{1}{2}\int_{\R^l}\big (\sqrt{\pi_1(x)}-\sqrt{\pi_2(x)} \big)^2\rd x\bigg)^{1/2}
\] 
satisfies $2 D_{\rm Hell}(\pi_1, \pi_2)^2 \leq D_{\rm KL}(\pi_1 , \pi_2)$. Recalling \eqref{eq:general_ent_diff} completes the proof.
\end{proof}

\begin{remark}
\label{rem:centered-KL-bound}
Because the differential entropy $\ent(\pi)$ is independent of the mean of~$\pi$, we can directly rephrase Lemma \ref{lem:entropy-KL-bound} by replacing $\pi_1$ and $\pi_2$ on the right-hand sides of \eqref{eq:diff-entr-chi2-ineq} and \eqref{eq:diff-entr-KL-alt-ineq} with their centered versions $\tilde{\pi}_1(z) = \pi_1(z - \E^{\pi_1} Z)$ and $\tilde{\pi}_2(z) = \pi_2(z - \E^{\pi_2} Z)$ to potentially improve the upper bounds. Be that as it may, in what follows we do not utilize this improvement.
\end{remark}

\begin{remark}
As Lemma \ref{lem:entropy-KL-bound} plays a key role in the analysis below, it is a relevant question whether it can be improved. In the light of Remark \ref{rem:centered-KL-bound}, we can make the following observation: Consider two one-dimensional normal distributions ${\mathcal N}(0,1)$ and ${\mathcal N}(0,\sigma^2)$, where $\sigma \not= 1$. Then, by simply evaluating the entropy difference and the KL divergence between these distributions, 
\begin{equation*}
  \frac{D_{\rm KL}({\mathcal N}(0,1), {\mathcal N}(0,\sigma^2))^p}{|\ent({\mathcal N}(0,1))-\ent({\mathcal N}(0,\sigma^2))|} = \frac{\left(\frac 1{2\sigma^2} -\frac 12 + \log \sigma\right)^p}{|\log \sigma|} \longrightarrow 0
\end{equation*}
for any $p>\frac 12$ as $\sigma$ tends to $1$. In consequence, an improved rate with a higher power of the KL divergence in \eqref{eq:diff-entr-KL-alt-ineq} is unavailable for these simple densities.
\end{remark}

Let us now integrate Proposition~\ref{prop:total_error} and Lemma~\ref{lem:entropy-KL-bound} into a uniform bound over the design domain for the evidence $J$ defined in \eqref{eq:J}; from this point on, we assume that the noise process in \eqref{eq:compmodel} is a zero-mean Gaussian with covariance matrix $\Gamma$. To that end, define the likelihood energy
\begin{equation}
	\label{eq:Phigauss}
	\Phi(x, y) = \frac 12 \norm{\calG(x) - y}_{\Gamma}^2,
\end{equation}
where the weighted norm is defined via $\norm{z}_{\Gamma}^2 = z^\top \Gamma^{-1} z$ for $z \in \R^{d}$, and the associated posterior normalization constant
\begin{equation}
    \label{eq:Z}
    Z(y) = \E^\mu \exp(-\Phi(X,y)).
\end{equation}
We denote by $\Phi_K$ and $Z_K$, respectively, the likelihood energy and the corresponding normalization constant for the Bayesian model corresponding to the surrogate forward operator ${\mathcal G}_K$.
Notice that $Z(y)$ and $\pi(y)$ (respectively,  $Z_K(y)$ and $\pi^K(y)$) coincide up to a universal positive multiplicative constant depending on $d$ and $\Gamma$. Furthermore, let us denote 
\begin{equation}
    \label{eq:def_deltaK}
    \delta_K = \sqrt{\E^\mu \norm{{\mathcal G}(X) - {\mathcal G}_K(X)}_\Gamma^2}
\end{equation}
for the standard deviation of the prior-predictive forward model approximation.

\begin{theorem}
\label{thm:theorem1}
Assume there exist constants $C_0,M_0,K_0>0$ such that
\begin{equation}
\label{eq:ass_sup_exp}
\E^{\rho_1} \log^2 \rho_2(Y)  \leq C_0
\end{equation}
for $\rho_1, \rho_2 \in \{ \pi, \pi^K, \pi_M^K\}$ and 
\begin{equation}
    \label{eq:ass_sup_KL}
    \delta_K \leq C_0, \qquad D_{\rm KL}\big(\pi_M^K, \pi^K \big) \leq C_0  
\end{equation}
for all $M>M_0$ and $K>K_0$. Then, there exists a constant $C>0$ such that
\begin{equation}
\label{eq:estimate}
    \E^{\otimes \pi_M^K}\big| J - \JMN^K\big|^2
    \leq C\Big(\delta_K^2 + D_{\rm KL}\big(\pi_M^K, \pi^K \big)+ \frac 1N \Big) 
\end{equation}    
for all $K>K_0$, $M>M_0$ and $N>0$.
\end{theorem}

\begin{proof}
The proof is based on Proposition~\ref{prop:total_error} and Lemma~\ref{lem:entropy-KL-bound}. Applying the triangle inequality to the entropy difference in \eqref{eq:total_error_decomposition} gives
\begin{align}
\label{eq:estimate2}
\E^{\otimes \pi_M^K}\big| J - \JMN^K\big|^2&\leq 
      2 \left(\ent(\pi) - \ent(\pi^K)\right)^2   \nonumber\\&  \kern-1cm + 2\left(\ent(\pi^K) - \ent(\pi_M^K)\right)^2 + \frac 1N  \V_{\pi_M^K} \big(\log(\pi_M^K(Y)) \big),
\end{align}
where the terms on the right-hand side can be bounded in the same order by those on the right-hand side of \eqref{eq:estimate}, as reasoned in the following.

Let $K>K_0$ and $M>M_0$. According to \cite[Lemma 3.8]{duong2023stability} and the discussion in the beginning of \cite[Section~4]{duong2023stability},
\begin{equation}
    \label{eq:evidence_KLdiff_upper_bound}
    D_{\rm KL}\big(\pi^K, \pi\big) \leq \E^\mu D_{\rm KL}\big(\pi^K(\, \cdot \, \mid X), \pi(\, \cdot \, \mid X)\big) = \frac 12 \delta_K^2 \leq \frac{1}{2} C_0^2,
\end{equation}
 which, in particular, means that $D_{\rm KL}(\pi^K, \pi)$ is bounded by a constant times its square root. Thus, combining \eqref{eq:evidence_KLdiff_upper_bound} with  \eqref{eq:diff-entr-KL-alt-ineq} and \eqref{eq:ass_sup_exp} induces the first term on the right-hand side of~\eqref{eq:estimate}. Due to \eqref{eq:ass_sup_KL}, the same line of reasoning on the KL terms in the estimate \eqref{eq:diff-entr-KL-alt-ineq} also applies to the second term on the right-hand side of \eqref{eq:estimate2}, which results in the second term on the right-hand side of \eqref{eq:estimate}. Finally, the validity of the third term on the right-hand side of \eqref{eq:estimate} immediately follows from  \eqref{eq:ass_sup_exp}. 
\end{proof}

\begin{remark}
\label{remark:cubature}
    As discussed in the introduction, rather than relying on the MC estimator~\eqref{eq:def-estimator}, the entropy $\ent(\pi_M^K)$ can also be approximated numerically using alternative deterministic or randomized cubature rules. That is, for the deterministic case, one could introduce 
    \begin{equation*}
    \widetilde{J}^K_{M,N} = Q_N\big(\pi^K_M  \log(\pi^K_M) \big),
    \end{equation*}
    where
    \begin{equation*}
        Q_N(f) = \sum_{n=1}^N w_n f(y_n)
    \end{equation*}
    for some cubature weights $w_n \in \R$ and nodes $y_n \in \R^d$. In this case, the estimation error can similarly be decomposed into three parts:
    \begin{align*}
   &\big| J -  \widetilde{J}^K_{M,N} \big| \leq 
   \big| \ent(\pi) - \ent(\pi^K) \big| \\&\kern1cm + 
   \big| \ent(\pi^K) -  \ent(\pi_M^K)\big|
   + \big|   \ent(\pi_M^K)- Q_N\big(\pi^K_M  \log(\pi^K_M) \big)  \big| \\[1mm]
   & \leq C \Big( \sqrt{\delta_K} + \sqrt{D_{\rm KL}(\pi_M^K, \pi^K)} \Big) + \big|   \ent(\pi_M^K)- Q_N\big(\pi^K_M \, \log(\pi^K_M) \big)  \big|
\end{align*}
for $K>K_0$, $M>M_0$ and $N>0$ under the assumptions of Theorem~\ref{thm:theorem1}. 

We will exploit this idea in our numerical experiments in order to get a higher convergence rate in $N$, which enables isolating the effect of $M$ in the convergence. To that end, suppose one can express $\pi^K_M \log(\pi^K_M) = f \rho$, where $\rho$ is a product of $d$ {\em monotonic Schwartz weights} (see \cite[Section~2.1]{SHK2025}) and $f$ belongs to the tensor product of the corresponding one-dimensional weighted $L^2$-based Sobolev spaces with smoothness index $\alpha \in \N$. Resorting to the component-wise change of variables $\Psi: X=(X^{(1)},\ldots,X^{(d)})\mapsto(\psi(X^{(1)}),\ldots,\psi(X^{(d)}))$, with $\psi(x)=-\cot(\pi x)$, we define
\[
Q_N^\Delta(f)=\sum_{n=1}^N \frac{|\prod_{j=1}^d \psi' (X_n^{(j)})|}{N} f(\Psi(X_n)),
\]
where $\{X_n\}_{n=1}^N$ corresponds to the rank-1 lattice rule defined in \eqref{eq:rand-lattice-def} for a $d$-dimensional setting. This quadrature, i.e.~a randomized M\"obius-transformed lattice rule, may achieve higher order convergence
\begin{equation*}
\E^\Delta \big|   \ent(\pi_M^K)- Q_N^\Delta(\pi^K_M \, \log(\pi^K_M) )  \big| \leq C_d \, \frac{(\log(N))^{d\alpha} }{N^{\alpha}},
\end{equation*}
where the expectation is with respect to the random shift in \eqref{eq:rand-lattice-def} (cf.~\cite[Theorem~4.1]{KSG2025}). However, apart from numerically testing the M\"obius-transformed lattice rule for constructing $\widetilde{J}^K_{M,N}$ in Section~\ref{sec:numerics}, we will not stress such a cubature-based approach any further in this work. 
\end{remark}

\section{MC based GMM evidence surrogate}\label{sec:GMM}

In this section, we construct a surrogate evidence $\pi_M^K$ as a Gaussian mixture formed as a push-forward through \eqref{eq:compmodel} of a randomized ensemble drawn from the prior. More precisely, we define
\begin{equation}
    \label{eq:GMM_definition}
    \pi_M^K(y) = \frac{1}{M}\sum_{m=1}^M \pi^K(y\mid X_m),
\end{equation}
where $X_m\sim \mu$, $m=1,\dots,M$, are i.i.d. and
\begin{equation}
\label{eq:K_GM}
\pi^K(y\mid X_m) = \frac{1}{\sqrt{(2 \pi)^d | \Gamma |}} \exp \! \big( -\Phi_K(X_m, y) \big), \qquad m=1,\dots,M.
\end{equation}

We now state the central assumption on the inverse problem that underpins the analysis in this section.

\begin{assumptionA}
\label{ass:prior_perturbation}
The forward operator $\calG : {\mathcal X} \to \R^{d}$ and the Borel probability measure $\mu$ on ${\mathcal X}$ satisfy the following conditions:
\begin{itemize}
	\item[(i)](uniformly Lipschitz continuous $\mathcal{G}$) There exists $L_1>0$ such that
	\begin{equation*}
		\norm{\calG(x) - \calG(x')}_\Gamma \leq L_1 \norm{x-x'}
	\end{equation*}
	for all $x,x' \in {\mathcal X}$.
	\item[(ii)](sub-Gaussian prior) There exists $L_2>0$ such that
	\begin{equation*}
		\E^\mu \exp\! \big(L_2 \norm{X}^2\big) < \infty.
	\end{equation*}
	\item[(iii)](proper $\calG$) There exist $x_0\in {\mathcal X}$ and $R, L_3>0$ such that
	$\mu(B(x_0,R))>0$ and\newline $\sup_{x\in B(x_0,R)} \norm{\calG(x)}_\Gamma < L_3.$
\end{itemize}
\end{assumptionA}

\begin{remark}
\label{rem:bounded_signal_to_noise}
In the main results of this section, we impose a relation between the parameters $L_1$ and $L_2$ in Assumption \ref{ass:prior_perturbation}, namely,
\begin{equation}
\label{eq:Lcondition}
L_1^2 < C L_2 ,
\end{equation}
for a certain constant $C>0$. This condition is reminiscent of the setting in \cite{GGNP2019}, where the authors establish convergence of smoothed empirical measures $\rho \star (\tfrac{1}{M}\sum_{m=1}^M \delta_{X_m})$ to $\rho \star \tilde \mu$, with $X_m \sim \tilde \mu$ i.i.d. and $\star$ denoting convolution, under the assumption of a bounded signal-to-noise ratio.

To highlight the connection, consider the case $\Gamma = \sigma^2 I$ with standard deviation $\sigma>0$. 
In our setting, $\tilde{\mu}$ from~\cite{GGNP2019} corresponds to the push-forward of $\mu$ under $\calG$. Now suppose $\calG$ is Lipschitz with constant $\alpha>0$ with respect to the Euclidean norm on the image space $\R^d$. Considering Assumption \ref{ass:prior_perturbation}, it follows that condition \emph{(i)} is satisfied with $L_1 = \frac \alpha \sigma$. 
Moreover, we have
\begin{equation*}
    \E^\mu \exp\!\big(\widetilde{L}\norm{\calG(X)}^2\big) \leq C \E^\mu \exp\!\big(L_2 \norm{X}^2\big) < \infty
\end{equation*}
for $\widetilde{L}=L_2/\alpha^2$. Therefore, for a fixed mapping $\calG$, the condition \eqref{eq:Lcondition} implies 
\begin{equation*}
    \left(\frac \alpha \sigma \right)^2 < C \alpha^2 \widetilde{L} \quad \Longrightarrow \quad \frac{1}{\sqrt{\widetilde{L}}} < C \sigma.
\end{equation*}
This inequality shows that the concentration of $\tilde \mu$, which increases with $\widetilde{L}$, imposes a lower bound on the noise level $\sigma$. In other words, condition \eqref{eq:Lcondition} also imposes a bound on the signal-to-noise ratio. More specifically, in \cite{GGNP2019} the convergence in expected KL divergence at rates comparable to Theorem \ref{thm:MC-GMM-estimator} is obtained under the condition $K < \sigma/2$, with $K$ quantifying the concentration of the sub-Gaussian distribution in the standard sense (i.e.,~smaller $K$ implying more concentration).
\end{remark}

The following lemma, which is the backbone of the analysis in this section, is the point where the assumed Gaussian likelihood structure becomes
essential: both the quadratic form of $\Phi$ and the Gaussian tails of the likelihood are used to obtain the upper and lower bounds on the evidence normalization constant.
For the proof, see Lemmas 3.10 and 3.11 in \cite{helin2025wasserstein}. 

\begin{lemma}[Basic properties]
\label{lem:Likeli_Phi_basic}
Let $\calG$ satisfy Assumption \ref{ass:prior_perturbation} for a probability measure $\mu$ on ${\mathcal X}$, and assume $\Phi$ is given by \eqref{eq:Phigauss}. Then for any $\tau>0$,
\begin{equation}
\label{eq:young_joint}
-\Phi(x;y) \leq -\frac{1-\tau}2 \norm{y}_\Gamma^2 + \frac{1-\tau}{\tau}L_1^2 \norm{x}^2 + C,
\end{equation}
where the constant $C>0$ depends on $\tau$, $R$ and $L_3$.
Moreover, for any $\kappa>\frac 12$, there exist finite constants $C', C''>0$ such that
\begin{equation}
	\label{eq:Z_equiv}
 C' \exp \! \big(-\kappa \norm{y}_\Gamma^2\big) \leq Z(y) \leq C'' \exp \! \Big(-\frac 12\frac {L_2}{L_1^2 + L_2}\norm{y}_\Gamma^2\Big)
\end{equation}
for any $y\in\R^{d}$, with $Z$ given by \eqref{eq:Z}.
\end{lemma}

Lemma \ref{lem:Likeli_Phi_basic} gives rise to the next two corollaries which are utilized in the proof of Theorem \ref{thm:MC-GMM-estimator}.

\begin{corollary}
\label{cor:finite_squared_log}
Suppose Assumption \ref{ass:prior_perturbation} holds uniformly with respect to $K$ for $\calG$ and $\calG_K$ with a probability measure $\mu$ on ${\mathcal X}$, and let $X_j\sim \mu$, $j=1,\dots,M$, be i.i.d.. 
For $\rho_1, \rho_2 \in  \{\pi, \pi^K \}$,
\begin{equation}
\label{eq:exp_rho_bound}
\E^{\rho_1} \! \log^2 \! \rho_2(Y) < \infty \qquad {\rm and} \qquad 
\E^{\otimes \mu} \E^{\pi_M^K}\log^2\pi_M^K(Y) < \infty,
\end{equation}
where the bounds are independent of $M$ and $K$.
\end{corollary}

\begin{proof}
Since $\Phi\geq 0$ everywhere, each considered marginal density is bounded from above by a constant $C(d,\Gamma)$. In particular, there exists another constant $C_\alpha>0$ such that for any $\alpha > 0$, 
\begin{equation}
\label{eq:alpha_eq}
    \log^2 x \leq C_\alpha x^{-\alpha} \quad \text{for any } x \in (0, C(d,\Gamma)].
\end{equation}
In consequence,
\begin{equation}
\label{eq:exp_log_rho}
    \E^\rho \log^2 \! \rho(Y) \leq C_\alpha \int_{\R^d} \rho(y)^{1-\alpha} \rd y
\end{equation}
for $\rho \in \{\pi, \pi^K, \pi_M^K\}$.
The left bound in \eqref{eq:exp_rho_bound} for $\rho_1 = \rho_2 = \pi$ (respectively, for $\rho_1 = \rho_2 = \pi^K$) now follows from \eqref{eq:Z_equiv} since $\pi$ and $Z$ (respectively, $\pi^K$ and $Z_K$) differ by a universal multiplicative constant.

To prove the assertion for $\rho_1 = \pi$ and $\rho_2 = \pi^K$, note that by \eqref{eq:alpha_eq} and \eqref{eq:Z_equiv} we have for any $\alpha > 0$ and $\kappa > 1/2$ that
\begin{align*}
    \E^{\pi} \log^2\pi^K(Y) &\leq C_\alpha  \E^{\pi} \big[ \pi^K(Y)^{-\alpha}\big] \\[1mm]
    &\leq C_{\alpha, \kappa}  \E^{\pi} \exp \! \big(\alpha\kappa\norm{Y}_\Gamma^2\big) \\[1mm]
    &\leq C_{\alpha, \kappa}' \int_{\R^d} \exp \!\left(\left(\alpha \kappa - \frac 12 \frac{L_2}{L_1^2 + L_2}\right)\norm{y}_\Gamma^2\right) {\rm d} y,
\end{align*}
which is finite if $\alpha \kappa>0$ is chosen to be small enough. The case  $\rho_1 = \pi^K$ and $\rho_2 = \pi$ follows by exactly the same argument. 

Consider next the second part of \eqref{eq:exp_rho_bound}. By the inequality \eqref{eq:young_joint} with $\tau = L_1^2/(L_2 + L_1^2)$, we have
\begin{equation}
\label{eq:no_name0}
    \pi_M^K(y) \leq C \exp \! \Big(-\frac 12 \frac{L_2}{L_1^2+L_2}\norm{y}_\Gamma^2\Big) \bigg(\frac 1M \sum_{m=1}^M \exp \! \big(L_2 \norm{X_m}^2\big) \bigg)
\end{equation}
for a constant $C$ that depends on $L_1$, $L_2$, $L_3$, $R$ and $d$. Resorting to Jensen's inequality with $\alpha \in (0,1)$ thus gives
\begin{align}
\label{eq:no_name}
    &\E^{\otimes \mu}  \left[\int_{\R^d} \pi_M^K(y)^{1-\alpha} \rd y \right] \\
    & \leq C \, \E^{\otimes \mu} \bigg(\frac 1M \sum_{m=1}^M \exp \! \big(L_2 \norm{X_m}^2\big) \bigg)^{1-\alpha} \! \! \int_{\R^d}\exp \! \bigg(-\frac{1-\alpha}{2} \frac{L_2}{L_1^2+L_2}\norm{y}_\Gamma^2\bigg) \rd y \nonumber \\
    & \leq C \bigg(\E^{\otimes \mu}\frac 1M \sum_{m=1}^M \exp\!\big(L_2 \norm{X_m}^2\big)\bigg)^{1-\alpha} = C \Big(\E^\mu \exp \big(L_2\norm{X}^2\big)\Big)^{1-\alpha} < \infty, \nonumber
\end{align}
where the last step follows from Assumption~\ref{ass:prior_perturbation}(ii) and the generic constant $C$, which depends on $L_1$, $L_2$, $L_3$, $R$, $d$ and $\alpha$, may differ between occurrences. Combining this with \eqref{eq:exp_log_rho} completes the proof.
\end{proof}

\begin{corollary}
\label{prop:likelihood_evidence_ratio}
    Suppose that $\mathcal{G}$ satisfies Assumption \ref{ass:prior_perturbation} for a probability measure $\mu$ on ${\mathcal X}$. Let $p > 1$ and assume $L_1^2 < \frac{1}{p(p-1)}L_2$. Then,
\begin{equation}
    \label{eq:like_evi_ratio_bound}
    \E^\pi \E^\mu \left|\frac{\exp(-\Phi(X,Y))}{Z(Y)}\right|^p < \infty.
\end{equation}    
\end{corollary}
\begin{proof}
Combining the inequalities \eqref{eq:young_joint} and \eqref{eq:Z_equiv} in Lemma \ref{lem:Likeli_Phi_basic}, we obtain
\begin{equation}
   \label{eqn:ratio_likelihood_evidence_key_deduction}
	\frac{\exp(-p \, \Phi(x;y))}{Z(y)^{p-1}} \leq 
	C \exp \! \Big(\frac{1}{2}\big( 2 \kappa (p-1) -(1-\tau) p \big) \norm{y}^2_\Gamma + \frac{1-\tau}{\tau} p L_1^2 \norm{x}^2\Big),
\end{equation}
where $\kappa>1/2$ and $\tau>0$ can be chosen arbitrarily, with their values only affecting the constant~$C$. Since $Z$ and $\pi$ differ by a multiplicative constant that depends (only) on $d$ and $\Gamma$, the finiteness of the expectation in \eqref{eq:like_evi_ratio_bound} thus follows by Assumption~\ref{ass:prior_perturbation}(ii) if there exist $\kappa_0>1/2$ and $\tau_0>0$ such that
\begin{equation}
\label{eq:conditions}
g(\tau_0, \kappa_0) := 2 \kappa_0 (p-1) - (1-\tau_0) p < 0 \quad \text{and} \quad f(\tau_0) := \frac{1-\tau_0}{\tau_0} p L_1^2 \leq L_2,
\end{equation}
which is what we will prove in what follows.

As $f : \R_+\to \R$ is continuous and decreasing and $f(\R_+) = (-pL_1^2,\infty)$, the second condition in \eqref{eq:conditions} is satisfied by every $\tau \geq \tau_0$, with $\tau_0$ defined as the unique solution of $f(\tau_0) = L_2$. Solving for~$\tau_0$, noting that the function $t \mapsto t/(t+1)$ is increasing, and utilizing our assumption on $L_1$ and $L_2$ yields
\begin{equation}
\label{eq:tau0}
    \tau_0 = \frac{pL_1^2}{pL_1^2 + L_2} <   \frac{\frac 1{p-1} L_2}{\frac 1{p-1} L_2 + L_2} = \frac 1p.
\end{equation}
Let us define
\begin{equation}
\kappa_0 = \frac{1}{2} \Big( 1 + \frac{1 - \tau_0 p}{2(p-1)}\Big) > \frac{1}{2}.
\end{equation}
A direct calculation reveals that for such choices,
\[
g(\tau_0, \kappa_0) = (p-1) + \frac{1}{2}(1 - \tau_0 p) - (1- \tau_0) p = \frac{1}{2}( \tau_0 p - 1) < 0
\]
by virtue of \eqref{eq:tau0}. This completes the proof.
\end{proof}

\begin{theorem}
\label{thm:MC-GMM-estimator}
Suppose Assumption \ref{ass:prior_perturbation} holds for $\mathcal{G}$ and $\mathcal{G}_K$ with $L_1^2 < \frac 1{2} L_2$ and a probability measure $\mu$ on $\mathcal{X}$, and let $X_j\sim \mu$, $j=1,\dots ,M$ be i.i.d. Moreover, assume there exists $C_0, K_0 > 0$ such that $\delta_K \leq C_0$ for $K > K_0$, where $\delta_K$ is given in \eqref{eq:def_deltaK}. Then,
\begin{equation}
\label{eq:theorem2}
    \E^{\otimes \mu} \E^{\otimes \pi_M^K}\big|J - \JMN^K\big|^2 \leq C\Big(\delta_K^2 + \frac 1M + \frac 1N\Big)
\end{equation}
for some constant $C$ and all $K > K_0$ and $N, M > 0$.
\end{theorem}

\begin{proof}

Let $K > K_0$. As in the proof of Theorem~\ref{thm:theorem1}, we write
\begin{align}
    \label{eq:GMM_thm_aux1}
    &\E^{\otimes \pi_M^K}  \big| J  - \JMN^K\big|^2 \\ 
     & \quad \leq 
      2 \big(\ent(\pi) - \ent(\pi^K)\big)^2 + 2\big(\ent(\pi^K) - \ent(\pi_M^K)\big)^2 + \frac 1N  \V_{\pi_M^K} \big(\log(\pi_M^K(Y)) \big). \nonumber
\end{align}
The $\otimes \mu$-expectation of the variance term in \eqref{eq:GMM_thm_aux1} is bounded by a constant due to Corollary~\ref{cor:finite_squared_log}, giving rise to the last term on the right-hand side of \eqref{eq:theorem2}. Furthermore, as in \eqref{eq:evidence_KLdiff_upper_bound},
\begin{equation*}
D_{\rm KL}(\pi^K,\pi) \leq \frac{1}{2}\delta_K^2 \leq \frac{1}{2} C_0^{2}.
\end{equation*}
Combining this with \eqref{eq:diff-entr-KL-alt-ineq} of Lemma~\ref{lem:entropy-KL-bound} and Corollary~\ref{cor:finite_squared_log} demonstrates that there exists a constant $C > 0$ such that
\begin{equation*}
    \big|\ent(\pi) - \ent(\pi^K)\big|\leq C \delta_K,
\end{equation*}
which results in the first term on the right-hand side of \eqref{eq:theorem2}. 

We complete the proof by bounding the $\otimes \mu$-expectation of the second term on the right-hand side of \eqref{eq:GMM_thm_aux1} with the help of~\eqref{eq:diff-entr-chi2-ineq} in Lemma~\ref{lem:entropy-KL-bound}. 
Denoting $W_m(y) = \pi^K(y \mid X_m)/\pi^K(y) - 1$, we have
\begin{align}
    \E^{\otimes \mu} \! \chi^2 (\pi^K_M,\pi^K)
    & = \E^{\pi^K}\E^{\otimes \mu} \bigg|\frac{1}{M}\sum_{m=1}^M W_m(Y)\bigg|^{2}
      = \frac{1}{M} \E^{\pi^K}\E^{\mu}  W_1(Y)^2 \nonumber \\
    & \leq \frac{2}{M} \bigg( \E^{\pi^K}\E^{\mu} \bigg|\frac{\pi^K(Y \mid X)}{\pi^K(Y)}\bigg|^2 + 1 \bigg),
    \label{eq:bound_for_chi2}
\end{align}
where the second equality holds because
the random variables $W_m(Y)$, $m=1,\dots,M$, are i.i.d.\ with $\E^\mu W_m(Y) = 0$.

By our assumption $L_1^2 < \frac 12 L_2$ and Corollary \ref{prop:likelihood_evidence_ratio} with $p=2$ (with $\mathcal{G}_K$ in place of $\mathcal{G}$), the expectation on the right-hand side of \eqref{eq:bound_for_chi2} is finite. Together with \eqref{eq:diff-entr-chi2-ineq} and Corollary~\ref{cor:finite_squared_log}, this leads to
\begin{equation*}
     \E^{\otimes \mu} \big(\ent(\pi^K) - \ent(\pi_M^K)\big)^2  \leq 2\big(  \E^{\pi^K} \! \log^2 \pi^K(Y) + 1 \big) \, \E^{\otimes \mu} \chi^2(\pi^K_M, \pi^K) \leq \frac{C}{M}
\end{equation*}
for $M  > 0$, which completes the proof.

\end{proof}

\section{QMC based GMM evidence surrogate}\label{sec:QMC-GMM}

This section develops a GMM estimator, following \eqref{eq:GMM_definition}, based on QMC points in a finite-dimensional subspace rather than samples from the prior. Our standing assumption is that the forward mapping $\calG$ and the prior $\mu$ satisfy Assumption~\ref{ass:prior_perturbation}.

Let ${\mathcal X}_K \subset {\mathcal X}$, $K \in \N$, be a subspace characterized by a projection and isomorphic to $\R^K$. We define the approximate forward mapping $\mathcal G_K$ on ${\mathcal X}_K$ and extend it canonically to the whole space. For example, one could consider an unconditional Schauder basis $\{\phi_j\}_{j=1}^\infty \subset {\mathcal X}$ giving rise to a sequence of nested subspaces ${\mathcal X}_K = \text{span}\{\phi_1, \ldots, \phi_K\}$.

We identify ${\mathcal X}_K$ with $\R^K$ and suppose that the approximation error $\delta_K$ given in \eqref{eq:def_deltaK}
can be controlled by adjusting $K$. Excluding $\delta_K$, the total error of the method depends on the marginal of $\mu$ on $\R^K$, which we denote, with a slight abuse of notation, again by $\mu$. In this section, we consider two types of prior distributions: first, a uniform distribution over a hypercube and, second, a Gaussian measure on $\R^K$. 
We emphasize that the parameter $K$ reflects not only the error arising from the finite-dimensional projection but also potential model and discretization errors. This aspect will be clarified in the numerical experiments.

The aim is to deduce estimates of the type \eqref{eq:theorem2}. To this end, we need to control the discrepancy between the evidence induced by $\calG_K$, i.e.~$\pi^K$, and the surrogate evidence 
\begin{equation}
    \label{eq:piM_QMC}
    \pi^K_M(y)=\frac{1}{M}\sum_{m=1}^M \pi^K(y \mid X_m),
\end{equation}
where $\{X_m\}_{m=1}^M$ are randomized QMC points and $\pi^K(y \mid X_m)$ is as defined in~\eqref{eq:K_GM}.

We employ \emph{randomly shifted rank-$1$ lattices} as our QMC point set.  For the uniform prior over $[0,1]^{K}$, the randomized lattice points are defined by three parameters,~i.e.~the generating vector $z$, the number of points $M$ and the random shift $\Delta$:
\begin{align}\label{eq:rand-lattice-def}
    X_{\Delta}:=\Big\{ X_m= \Big( \frac{z m}{M} + \Delta \bmod{1} \Big) \; \Big|  \; m=1,\ldots,M \Big \},
\end{align}
where the components of the random shift are chosen uniformly, i.e.~$\Delta \sim U([0,1]^{K})$, ``mod $1$'' takes the fractional part of a real number, and the generating vector is a carefully chosen integer vector from $\Z_M^{K}:=\{0,1,\dots,M-1\}^{K}$. 
Here, the random shift $\Delta$ is fixed for all $m=1,\dots ,M$. 
When we consider the Gaussian prior, the lattice is mapped with the component-wise inverse transform of the cumulative density function $\Psi_{\text{CDF}}^{-1}:(0,1)^{K} \to \R^{K}$ for the multivariate standard Gaussian distribution, that is, we define
\begin{align}\label{eq:rand-lattice-def2}
    \widetilde{X}_{\Delta}:= \Psi_{\text{CDF}}^{-1}({X}_{\Delta}).
\end{align}
Randomly shifted lattices have been popular in the context of integration over unbounded domains in recent years; see, e.g., \cite{NK2014,GKNSSS2015,HKS2021} for more information. The generalization of our approach for higher-order QMC methods will be considered in Remark~\ref{rem:ho-qmc} for the uniform prior.

The reason we employ randomly shifted lattice rules is two-fold: (i) In order to compare the results by QMC with those by MC using the same error criterion, namely the RMSE, we employ random shifting rather than interpret deterministic QMC rules as a special case of randomized algorithms. (ii) To be able to use the algorithms for unbounded domains,~i.e.,~for a Gaussian prior in Section~\ref{subsec:QMC/Gaussian_prior}, we need randomization to obtain a theoretical error bound. The randomization also helps to avoid placing QMC points on the boundary of $[0,1]^K$, where the value of $\Psi_{\text{CDF}}^{-1}$ is not defined.

\subsection{Uniform prior}
\label{sec:QMC_uniform}
Let $\mu$ be the uniform prior on the unit cube $[0,1]^{K}$, i.e., $\mu(\rd x) = {\bf 1}_{[0,1]^{K}}(x) \, \rd x,$
where ${\bf 1}$ denotes the characteristic function of the indicated set. Under our likelihood model, the posterior is defined via
\[
\mu^y(\rd x) = \frac{1}{Z_K(y)} \exp(-\Phi_K(x,y)) {\bf 1}_{[0,1]^{K}}(x) \, \rd x =: \rho_K(x \mid y) \, \rd x,
\]
where $\Phi_K$ and $Z_K$ are given by \eqref{eq:Phigauss} and \eqref{eq:Z}, respectively, with $\mathcal{G}_K$ replacing $\mathcal{G}$.

For $\alpha \in \N$, we define the associated $L^2$- and $L^\infty$-based Sobolev spaces of {\em dominating mixed smoothness} by setting
\begin{equation*}
    \norm{f}^2_{W^{\alpha,2}_{\text{mix}}([0,1]^{K})} =
    \sum_{|\tau|_{\infty}\le\alpha} \int_{[0,1]^{K}}\Big|\frac{\partial^{|\tau|}}{\partial x^{\tau}} f\Big|^2 \rd x,
\end{equation*}
and 
\begin{equation*}
    \norm{f}_{W^{\alpha,\infty}_{\text{mix}}([0,1]^{K})} =
    \max_{|\tau|_{\infty}\le\alpha}  \, {\rm ess} \!\!\! \! \sup_{x \in [0,1]^{K}} \Big|\frac{\partial^{|\tau|}}{\partial x^{\tau}} f (x) \Big|, 
\end{equation*}
where $\tau \in \N_0^K$ denotes a multi-index for which $|\tau|=\tau_1+\ldots+\tau_K$ and $|\tau|_{\infty}=\max\{\tau_1,\ldots,\tau_K\}$. 
More precisely, the spaces $W^{\alpha,2}_{\text{mix}}([0,1]^{K})$ and $W^{\alpha,\infty}_{\text{mix}}([0,1]^{K})$ consist of those measurable functions on $[0,1]^{K}$ for which the respective norms are well-defined and finite.

\begin{lemma}
\label{lem:QMC_uniform_lemma}
Suppose $\calG_K \in W^{1,\infty}_{\rm mix}([0,1]^{K})^d$. Then the pair $\calG_K$ and $\mu$ satisfies Assumption~\ref{ass:prior_perturbation}. Moreover, for any $\tau>0$ there exists a constant $C_{K,\tau}$ such that
\begin{equation}
    \label{eq:QMC_PhiK_bound}
    -\Phi_K(x,y) \leq -\frac{1-\tau}{2} \norm{y}_\Gamma^2 + C_{K,\tau} \quad \text{and} \quad 
    Z_K(y) \leq C_{K,\tau} \exp\left(-\frac{1-\tau}{2}\norm{y}_\Gamma^2\right)
\end{equation}
for all $x\in [0,1]^{K}$ and $y\in \R^d$. In addition, 
\begin{equation}
    \label{eq:QMC_log2_bound}
    \E^\Delta \E^\rho \log^2 \rho(Y) < \infty
\end{equation}
for $\rho \in \{\pi^K, \pi_M^K\}$, as well as for $\rho = \pi$ if $\mathcal{G}$ satisfies Assumption~\ref{ass:prior_perturbation} with the original (i.e., non-projected) prior $\mu$.
\end{lemma}

\begin{proof}
Under the presented conditions, $\calG_K$ and $\mu$ satisfy Assumption~\ref{ass:prior_perturbation} for some $L_1$ and $L_3$ and any $L_2$. The part (ii) of the assumption follows trivially. The other two parts are straightforward consequences of the following observation: a function in $W^{1,\infty}_{\text{mix}}([0,1]^{K})^d\subset  W^{1,\infty}([0,1]^{K})^d $ on a (quasi)convex domain $[0,1]^{K}$ is Lipschitz on $[0,1]^{K}$~\cite[Theorem~4.1]{Heinonen05}, and by Kirszbraun’s theorem, it can be extended as a Lipschitz continuous function to the whole of $\R^{K}$, with the Lipschitz constant remaining the same \cite[Theorem~2.5]{Heinonen05}. 

The left bound in \eqref{eq:QMC_PhiK_bound} follows directly from Young's inequality since $x$ belongs to a bounded set, and
the log-bound \eqref{eq:QMC_log2_bound} follows from the same line of reasoning as Corollary \ref{cor:finite_squared_log} since the upper bound 
\begin{equation*}
    \pi^K_M(y) \leq C_{K, \tau} \exp\left(-\frac{1-\tau}{2}\norm{y}_\Gamma^2\right)
\end{equation*}
is independent of the employed cubature. Finally, the bound for $Z_K$ in \eqref{eq:QMC_PhiK_bound} follows by replacing the empirical measure with $\mu$.
\end{proof}

\begin{proposition}
\label{prop:dom_norm}
If $\calG_K \in W^{\alpha,\infty}_{\rm mix}([0,1]^{K})^d$ for $\alpha \in \N$, then there exists $b>0$ such that
\begin{equation*}
    \E^{\pi^K} \! \Big[\exp\!\big(b\norm{Y}_\Gamma^2\big) \norm{\rho_K(\, \cdot \,  \mid Y)}_{W^{\alpha,2}_{\rm mix}([0,1]^{K})}^2\Big] < \infty.
\end{equation*}
\end{proposition}

 \begin{proof}

Assume $\calG_K \in W^{\alpha,\infty}_{\rm mix}([0,1]^{K})^d$ for $\alpha \in \N$, and let $\tau \in \N_0^K$ be such that $| \tau |_\infty \leq \alpha$. According to Corollary~\ref{eq:alpha_deriv_bound},
 \begin{equation*}
\Big| \frac{\partial^{|\tau|}}{\partial x^{\tau}}  \rho_K(x \mid y) \Big|^2 \leq \frac{1}{Z_K(y)^2} \exp(-2\Phi_K(x,y)) \, q_\tau (y), \qquad x \in (0,1)^K, \ y \in \R^d, 
\end{equation*}
where $q_\tau: \R^d \to \R$ is a polynomial of degree $2|\tau|$.

 Recall that $\pi^K$ and $Z_K$ differ by a positive multiplicative constant. As in the proof of Corollary~\ref{prop:likelihood_evidence_ratio}, we can thus combine \eqref{eq:QMC_PhiK_bound} with the lower bound in \eqref{eq:Z_equiv} to deduce
\begin{align*}
\Big| \frac{\partial^{|\tau|}}{\partial x^{\tau}} \rho_K(x \mid y) \Big|^2 \! \pi_K(y) 
& \leq \frac{C'}{Z_K(y)} \exp \! \big( -(1-\tau)  \norm{y}^2_\Gamma\big) \, q_{\tau}(y) \\[1mm]
& \leq C'' \exp\! \big((\kappa+\tau-1) \norm{y}^2_\Gamma\big) \, q_{\tau}(y),
\end{align*}
where we can choose $\kappa > 1/2$ and $\tau>0$ such that $b := (1-\kappa-\tau)/2 > 0$, and the constant $C''$ depends on these choices.
Hence,
\begin{align*}
\E^{\pi^K} \! \Big[\exp\!\big(b\norm{Y}_\Gamma^2\big) &\norm{\rho_K(\, \cdot \,  \mid Y)}_{W^{\alpha,2}_{\rm mix}([0,1]^{K})}^2 \Big] \\[1mm]
& \leq C'' \sum_{|\tau|_\infty \leq \alpha} \int_{\R^d} \exp \! \big( -b  \norm{y}^2_\Gamma \big) \, q_{\tau}(y) \, \rd y < \infty
\end{align*}
due to the domination of the exponential part of the integrand.
 \end{proof}

Through standard QMC argumentation, the above proposition leads to convergence of $\pi^K_M$ towards $\pi^K$ in the expected $\chi^2$-distance and thus also in terms of the expected KL divergence (cf.~\eqref{eq:KL_bounded_by_chi2}), as revealed by the following proposition with $p=2$. We only consider the smoothness $\alpha = 1$ in Proposition~\ref{prop:E1-QMC-uniform} and Theorem~\ref{thm:QMC-uniform-result} due to the employment of a first order QMC method,~i.e.,~randomly shifted rank-$1$ lattices; see Remark~\ref{rem:ho-qmc} for information on combining higher regularity of $\calG_K$ with higher-order QMC methods.

\begin{proposition}\label{prop:E1-QMC-uniform}
Assume $\calG_K \in W^{1,\infty}_{\rm mix}([0,1]^{K})^d$.
Let $\{X_m\}_{m=1}^M$ be the randomized lattice points defined in \eqref{eq:rand-lattice-def} with the generating vector $z$ constructed by the component-by-component algorithm \cite[Algorithm~7]{K2003}, and let $\pi^K_M$ be as defined in \eqref{eq:piM_QMC}. Then, for any $p\geq 2$ and $\gamma>0$,
 \begin{align}
 \E^{\Delta} \E^{\pi^K} \Big|\frac{\pi^K _M(Y)}{\pi^K(Y)}-1\Big|^p
   \le   \frac{C }{M^{2-\gamma}}, \qquad M \in \N,
\label{eqn_quad_boud1}
\end{align}
where the constant $C>0$ depends on $K$, $\gamma$ and $p$.
\end{proposition}
\begin{proof}
To begin with, note that $\pi^K _M(y)$ is a randomized cubature rule for integrating $\pi^K(y \, \mid \, x)$ with respect to $x$,
and thus
\begin{equation}
\label{eq:uni_quad}
Q_M^\Delta(\rho_K(\, \cdot \, \mid y)) := \frac{\pi^K_M(y)}{\pi^K(y)} = \frac{1}{M} \sum_{m=1}^M \frac{\pi^K(y \mid X_m)}{\pi^K(y)} 
=
\frac{1}{M} \sum_{m=1}^M \rho_K( X_m  \mid y)
\end{equation}
is in turn a randomized cubature approximation for the integral of $\rho_K( \, \cdot \,  \mid y)$ over $[0,1]^{K}$, which evaluates to 1. Hence, for $p \geq 2$,
\begin{equation}
\label{eq:QMCquadrature}
\E^{\Delta} \E^{\pi^K}\Big|\frac{\pi^K _M(Y)}{\pi^K(Y)}-1\Big|^p 
 =  \E^{\Delta} \E^{\pi^K}\bigg|Q_M^\Delta(\rho_K(\, \cdot \, \mid Y)) -  \int_{[0,1]^{K}} \rho_K(x \mid Y) \, \rd x \bigg|^p,
\end{equation}
and our aim is to prove the claim by providing a suitable estimate for the right-hand side of \eqref{eq:QMCquadrature}.

By virtue of \eqref{eq:QMC_PhiK_bound} and \eqref{eq:Z_equiv} with $\tau = \tau' > 0$ and $\kappa = \kappa' > 1/2$, 
\begin{equation*}
    0 \leq \rho_K(x \mid y) = \frac{1}{Z_K(y)} \exp(-\, \Phi_K(x,y)) \leq C_{b'} \exp \!\big( b' \norm{y}_\Gamma^2 \big), \qquad x \in [0,1]^{K},
\end{equation*}
where the constant $C_{b'}$ depends on $b' := \kappa' - (1-\tau')/2 > 0$ that can be chosen to be arbitrarily small. Since $\rho_K(\, \cdot \, \mid y)$ is continuous,
\begin{equation}
    \label{eq:QMC_error_aux2}
    \bigg|Q_M^\Delta(\rho_K(\, \cdot \, \mid y)) -  \int_{[0,1]^{K}} \rho_K(x \mid y) \, \rd x \bigg|
    \leq 2 \norm{\rho_K( \, \cdot \, \mid y)}_{L^\infty([0,1]^{K})} \leq C_{b'}' \exp(b'\norm{y}_\Gamma^2),
\end{equation}
which holds uniformly with respect to $\Delta$.

Using a generating vector $z$ constructed by the component-by-component algorithm \cite[Algorithm~7 and Theorem~8]{K2003} and resorting to~\cite[Theorem~3.2]{SKJ2002}, we get
\begin{equation}
\label{eq:QMC_error_aux3}
\E^{\Delta} \bigg|Q_M^\Delta(\rho_K(\, \cdot \, \mid y)) -  \int_{[0,1]^{K}} \rho_K(x \mid y) \, \rd x \bigg|^2
\le
C_{K,\gamma}\frac{\|\rho_K(\, \cdot \,  \mid y)\|^{2}_{ W^{1,2}_{\text{mix}}([0,1]^{K})}}{M^{2-\gamma}},
\end{equation}
where $\gamma > 0$.
Combining
\eqref{eq:QMCquadrature}, \eqref{eq:QMC_error_aux2} and \eqref{eq:QMC_error_aux3}  yields
\begin{align}
    \label{eq:QMC_error_aux1}
&\E^{\Delta} \E^{\pi^K}\Big|\frac{\pi^K _M(Y)}{\pi^K(Y)}-1\Big|^p \!
\\& \leq  C''_{b'} \, \E^{\pi^K} \!   \Bigg[\! \exp \! \big((p-2) b'\norm{Y}_\Gamma^2\big) \, \E^\Delta \bigg|Q_M^\Delta(\rho_K(\, \cdot \, \mid Y)) -  \int_{[0,1]^{K}} \rho_K(x \mid Y) \, \rd x \bigg|^2\Bigg]\nonumber\\[1mm]
& \le   \frac{C_{K,\gamma, b'}}{M^{2-\gamma}}\E^{\pi^K} \!\!\left[ \exp \! \big((p-2) b'\norm{Y}_\Gamma^2\big) \, \|\rho_K(\, \cdot \,  \mid Y)\|_{W^{1,2}_{\text{mix}}([0,1]^{K})}^2\right] .
\end{align}
The assertion finally follows from  Proposition \ref{prop:dom_norm} by choosing a small enough $b' > 0$ such that $(p-2) b' \leq b$.
\end{proof}

Now, we are ready to complete the analysis for the uniform prior by proving the convergence rate for the randomized QMC-based surrogate.

\begin{theorem}
\label{thm:QMC-uniform-result}
Assume the projection of $\mu$ to $\calX_K$ is the uniform measure over $[0,1]^{K}$, the corresponding approximative forward operator satisfies $\calG_K \in W^{1,\infty}_{\rm mix}([0,1]^{K})^d$, $\delta_K$ given in \eqref{eq:def_deltaK} is bounded, and $\mathcal{G}$ satisfies Assumption~\ref{ass:prior_perturbation} with the original (i.e., non-projected) prior $\mu$.
Let $\{X_m\}_{m=1}^M$ be the randomized lattice points defined in \eqref{eq:rand-lattice-def} with the generating vector $z$ constructed by the component-by-component algorithm \cite[Algorithm~7]{K2003}, and let $\pi^K_M$ be as defined in \eqref{eq:piM_QMC}. Then,
\begin{equation}
\label{eq:QMC_rate}
\E^\Delta \E^{\otimes \pi_M^K}\big| J - \JMN^K \big|^2
    \leq C \bigg(\delta_K^2 + \frac{1}{M^{2-\gamma}} + \frac{1}{N} \bigg),
    \end{equation}
    where the constant $C > 0$ depends on $K$ and $\gamma > 0$. 
\end{theorem}

\begin{proof}
The assertion follows from a similar line of reasoning as Theorem~\ref{thm:MC-GMM-estimator}. Indeed, following \eqref{eq:GMM_thm_aux1}, we decompose the expected total squared error into three parts:
\begin{align}
    \label{eq:GMM_thm_aux1_QMC}
    \E^\Delta\E^{\otimes \pi_M^K} \big|  J -  \JMN^K\big|^2 \nonumber &\leq 
      2 \left(\ent(\pi) - \ent(\pi^K)\right)^2 \\& \kern-0.5cm + 2 \, \E^\Delta \! \left(\ent(\pi^K) - \ent(\pi_M^K)\right)^2 + \frac{1}{N} \,   \E^\Delta \V_{\pi_M^K} \big(\log(\pi_M^K(Y)) \big)
    \nonumber\\
    & \leq C \delta_K^2 + \frac{C'}{M^{2-\gamma}}   +     \frac{C''}{N},
\end{align}
where the second step follows by applying \eqref{eq:diff-entr-KL-alt-ineq},
\eqref{eq:evidence_KLdiff_upper_bound} and Corollary~\ref{cor:finite_squared_log} to the first term on the right-hand side, \eqref{eq:diff-entr-chi2-ineq}, \eqref{eqn_quad_boud1} and \eqref{eq:QMC_log2_bound} to the second term, and \eqref{eq:QMC_log2_bound} to the third term.
\end{proof}

\begin{remark}[Higher-order QMC]\label{rem:ho-qmc}
    The second term on the right-hand side of \eqref{eq:QMC_rate}, i.e.,~$1/M^{2-\gamma}$,  originates from \eqref{eqn_quad_boud1} with $p=2$.  Under the assumption $\calG_K \in W^{1,\infty}_{\rm mix}([0,1]^{K})^d$, the estimate \eqref{eqn_quad_boud1} is in turn  induced by the convergence rate of the (randomized) quadrature rule defined in \eqref{eq:uni_quad} in terms of the randomly shifted rank-$1$ lattice points \eqref{eq:rand-lattice-def}. In consequence, if one replaces the quadrature points in Theorem~\ref{thm:QMC-uniform-result} --- and thus also in the definition of the employed quadrature rule~\eqref{eq:uni_quad} --- by some other (randomized) QMC points such that, say, 
    \begin{equation}
    \label{eq:ho_QMC_estimate}
\E^{\Delta} \bigg|Q_M^\Delta(\rho_K(\, \cdot \, \mid y)) -  \int_{[0,1]^{K}} \rho_K(x \mid y) \, \rd x \bigg|^2
\le
C \frac{\|\rho_K(\, \cdot \,  \mid y)\|^{2}_{ W^{\alpha,2}_{\text{mix}}([0,1]^{K})}}{M^{2\alpha-\gamma}}
\end{equation}
instead of the basic rate $\alpha = 1$ in \eqref{eq:QMC_error_aux3}, then the second term on the right-hand side of \eqref{eq:QMC_rate} becomes $1/M^{2 \alpha-\gamma}$ under the assumption $\calG_K \in W^{\alpha,\infty}_{\rm mix}([0,1]^{K})^d$, cf.~Proposition~\ref{prop:dom_norm}. 
For example, one could use randomized tent-transformed lattice rules to achieve the convergence rate with $\alpha = 2$~\cite{H2000} or alternatively its deterministic counterpart~\cite{GSY2019}, or higher-order digital nets~\cite{GSY2018} for larger $\alpha$. We note that higher-order QMC points need to be constructed for each desired $\alpha$ in addition to requiring more smoothness of $\calG_K$. In Section~\ref{sec:numerics}, we numerically demonstrate second order convergence,~i.e.,~the convergence rate induced by $\alpha = 2$ in \eqref{eq:ho_QMC_estimate},  by the randomized tent-transformed lattice rule.
\end{remark}

\subsection{Gaussian prior}
\label{subsec:QMC/Gaussian_prior}
In this subsection, we suppose $\mu$ has white noise statistics on $\R^K \cong \mathcal{X}_K$ leading to the posterior
\[
\mu^y(\rd x) = \frac{1}{ Z_K(y)} \exp \! \big(-\Phi_K(x,y) \big)  \, \mu({\rm d}x) =: \sigma_K(x \mid y) \, \mu(x) \, {\rm d}x,
\]
where $\mu: \R^K \to \R_+$ denotes the standard Gaussian density. Take note that other types of Gaussian priors can also be presented in this form after a reparametrization based on a whitening/coloring transform.

To be able to present convergence rates, we define a function space with the norm
\begin{align*}
  & \|f\|_{W^{1,2}_{*}(\R^{K})}^2
\\&\quad := 
\sum_{\fraku \subseteq \mathcal{K} }\int_{\R^{|\fraku|}}\bigg(\int_{\R^{K-|\fraku|}} \frac{\partial^{|\fraku|}}{\partial  x_{\fraku}} f\left( x_{\fraku} ; x_{-\fraku}\right) \prod_{j \in -\fraku} \mu \left(x_j\right) \rd  x_{-\fraku}\bigg)^2 \prod_{j \in \fraku} \psi^2\left(x_j\right) \, \rd x_{\fraku},
\end{align*}
where $\mathcal{K} = \{1, 2, \ldots, K \}$, $|\fraku|$ stands for the cardinality of $\fraku \subset \mathcal{K}$, we denote $-\fraku = \mathcal{K} \setminus \fraku$ and the \emph{weight function} $\psi$ converges to zero slower than $\mu$ at infinity. For the precise assumptions on $\psi$, consult \cite[Eqs.~(9) and (10)]{KSWWW2010} and \cite[Table~1]{NK2014}. In our setting, one may choose, e.g.,  $\psi(x_j) \propto e^{-\alpha |x_j|}$ for some $\alpha>0$.  

\begin{theorem}
    \label{prop:E1-QMC-Gauss}
Assume $\mathcal{G}$ and $\mathcal{G}_K$ satisfy Assumption~\ref{ass:prior_perturbation} with the original (i.e., non-projected) prior and the projected prior, respectively. Furthermore, suppose $\mathcal{G}_K$ is regular enough so that
\begin{equation}
\label{eq:C_sigma}
C_\sigma := \E^{\pi^K} \! \norm{\sigma_K(\, \cdot \,  \mid Y)}_{W^{1,2}_{*}(\R^{K})}^2  < \infty.
\end{equation}
Let $\{X_m\}_{m=1}^M$ be the transformed randomized lattice points defined by \eqref{eq:rand-lattice-def2} with the generating vector z constructed by the
component-by-component algorithm \cite[Algorithm~6]{NK2014} and let $\pi^K_M$ be as in~\eqref{eq:piM_QMC}.
 Then, for any $\gamma>0$,
 \begin{align}
 \E^{\Delta}\chi^2 \big(\pi^{K} _M,\pi^K \big)
   \leq
 C \frac{C_\sigma}{M^{2-\gamma}},
\label{eqn_quad_boud2}
\end{align}
where the constant $C$ depends on $K$, $\gamma$ and $\psi$. Furthermore, 
\begin{equation}
\label{eq:log_variance_QMC_Gauss}
  \E^\Delta \E^\rho \log^2 \rho(Y) < \infty
\end{equation}
for $\rho \in \{\pi^K, \pi_M^K\}$,
and consequently,
\begin{equation}
\label{eq:QMC-Gauss}
\E^\Delta\E^{\otimes \pi_M^K}\big| J - \JMN^K \big|^2
    = \mathcal{O} \bigg(\delta_K^2 +  \frac{1}{M^{2-\gamma}} + \frac{1}{N} \bigg).
    \end{equation}

\end{theorem}

\begin{proof}
    The general argument for establising \eqref{eqn_quad_boud2} is exactly the same as in the proof of Proposition~\ref{prop:E1-QMC-uniform}.
    Indeed, denoting
    \begin{equation*}
    Q_M^\Delta(\sigma_K(\,\cdot \, \mid  y)) := \frac{\pi^K_M(y)}{\pi^K(y)} = \frac{1}{M} \sum_{m=1}^M \frac{\pi^K(y \mid X_m)}{\pi^K(y)},
    \end{equation*}
    the error bound
    \begin{align*}
    \E^{\Delta}\Big|\frac{\pi^{K}_M(y)}{\pi^K(y)}-1\Big|^2
    &=
    \E^{\Delta}\bigg| Q_M ^{\Delta}(\sigma_K(\, \cdot \, \mid y)) - \int_{\R^{K}} \sigma_K(x \mid y) \, \mu(\rd x) \bigg|^2
    \\&\le
    C_{K,\gamma,\psi} \frac{\|\sigma_K(\, \cdot \,  \mid y)\|^2_{W^{1,2}_*(\R^{K})}}{M^{2-\gamma}}
    \end{align*}
    follows from \cite[Theorem~8]{NK2014}. Thus,
    \[
    \E^{\Delta}\chi^2(\pi^{K} _M,\pi^K) = \E^{\pi^K} \E^{\Delta}\Big|\frac{\pi^{K}_M(Y)}{\pi^K(Y)}-1\Big|^2 \leq  C_{K,\gamma,\psi} \frac{C_\sigma}{M^{2-\gamma}},
    \]
    which completes the proof of \eqref{eqn_quad_boud2}. 

Let us then move on to proving \eqref{eq:log_variance_QMC_Gauss}. The result for $\rho = \pi^K$ follows as in Corollary~\ref{cor:finite_squared_log} since $\pi^K$ is independent of $\Delta$. As $\mathcal{G}_K$ satisfies Assumption~\ref{ass:prior_perturbation}, the surrogate $\pi^K_M$ satisfies \eqref{eq:exp_log_rho} and \eqref{eq:no_name0}. Hence, following the line of reasoning in \eqref{eq:no_name} for $\alpha \in (0,1)$,
\begin{align}
\label{eq:exp_delta_gauss}
\E^\Delta \E^{\pi^K_M} \log^2 (\pi^K_M(Y))  &\leq C \,
\E^{\Delta} \bigg(\frac{1}{M} \sum_{m=1}^M \exp \! \big(L_2 \norm{X_m}^2\big) \bigg)^{1-\alpha}  \nonumber \\[1mm]
& \leq C \bigg(
 \frac{1}{M} \sum_{m=1}^M \E^{\Delta} \exp \! \big(L_2 \norm{X_m}^2\big) \bigg)^{1-\alpha} ,
\end{align}
where the second step follows from Jensen's inequality and the transformed randomized lattice points $\{X_m\}_{m=1}^M$ depend on $\Delta$ via \eqref{eq:rand-lattice-def2} and \eqref{eq:rand-lattice-def}. In particular, each term in the sum on the right-hand side of \eqref{eq:exp_delta_gauss} satisfies
\begin{align*}
\E^{\Delta} \exp \! \big(L_2 \norm{X_m}^2\big) &= 
\int_{[0,1]^K} \exp \! \bigg(L_2 \norm{\Psi_{\text{CDF}}^{-1}\bigg(\frac{mz}{M}+\Delta \bmod{1} \bigg)}^2\bigg) \, {\rm d}  \Delta \\[1mm]
& = \int_{[0,1]^K} \exp \! \Big(L_2 \big\| \Psi_{\text{CDF}}^{-1}(w) \big\|^2 \Big) \, {\rm d} w \\[1mm]
& = \E^\mu \exp \big(L_2\norm{X}^2\big),
\end{align*}
where the last two equalities correspond, respectively, to the changes of variables $w=\big(\frac{mz}{M}+\Delta\mod{1}\big)$ and $x=\Psi_{\text{CDF}}^{-1}(w)$. 
The first change of variables has been previously used in QMC literature,~e.g.,~in \cite[Proof of Theorem~5.13]{DKS2013}.
Plugging this in \eqref{eq:exp_delta_gauss} and utilizing Assumption~\ref{ass:prior_perturbation} proves \eqref{eq:log_variance_QMC_Gauss} for $\rho = \pi^K_M$.

Finally, \eqref{eq:QMC-Gauss} follows from the reasoning employed in \eqref{eq:GMM_thm_aux1_QMC} with the applications of \eqref{eq:QMC_log2_bound} and \eqref{eqn_quad_boud1} replaced, respectively, by those of \eqref{eq:log_variance_QMC_Gauss} and \eqref{eqn_quad_boud2}; in particular, due to the form of \eqref{eq:diff-entr-chi2-ineq}, the second term on the right-hand side of \eqref{eq:GMM_thm_aux1_QMC} is controlled by the second moment bound \eqref{eqn_quad_boud2} alone.
\end{proof}

\section{Numerical experiments}\label{sec:numerics}

 In this section, we present two numerical experiments to demonstrate our method for estimating differential entropy, with the main aim to isolate the convergence rate in $M$. The procedure for computing a single realization of the estimator~\eqref{eq:def-estimator} is described in Algorithm~\ref{alm_gmm}. 
 The first numerical example is a linear problem with a Gaussian prior, with a known analytic form for the differential entropy of the associated evidence distribution. The second example considers a nonlinear PDE-based model with a high-dimensional uniform prior. Both experiments verify the presented convergence rates in $M$ even with relatively small sample sizes. 

 In both examples, we aim to isolate the convergence rate in $M$ by choosing $N$ sufficiently large for the $M$-dependent terms to dominate in the convergence estimates of Theorems~\ref{thm:MC-GMM-estimator}, \ref{thm:QMC-uniform-result} and \ref{prop:E1-QMC-Gauss}. As mentioned in Section~\ref{sec:intro}, the motivation for our interest in numerically verifying the convergence rates in $M$ are settings where the $M$ evaluations of the forward operator $\mathcal{G}_K$ in the first loop of Algorithm~\ref{alm_gmm} are so expensive that they overshadow the random draws and evaluations of $M$-mode GMM surrogate in the latter loop even for a very large $N$. Recall also that, as pointed out in Remark~\ref{remark:cubature}, the MC scheme for evaluating the differential entropy of the surrogate evidence in the second loop of Algorithm~\ref{alm_gmm} is not the only option, but one should choose the most appropriate entropy estimation method for the considered application to minimize the computational cost. In particular, as both of our numerical examples study relatively simple forward operators, we do not claim that the first loop of Algorithm~\ref{alm_gmm} dominates computation times in our examples for the considered large values of $N$; our numerical experiments should rather be considered as mere verifications of the convergence rates in $M$ in two simple settings.

\begin{algorithm}[t]
 	\caption{A realization of the estimator \eqref{eq:def-estimator} by MC or randomized QMC}\label{alm_gmm}
 	\begin{algorithmic}[1]
    \State {For randomized QMC, draw $\Delta$ from $U([0,1]^{K})$};
 		 \For{$m=1,\dots,M$} 
 		\State {Generate $x_m$ by drawing from the prior (MC) or via \eqref{eq:rand-lattice-def} or \eqref{eq:rand-lattice-def2} with the random shift $\Delta$ (QMC);}
 		\State {Evaluate the forward map: $z_m = \mathcal{G}_K(x_m)$;}

 		\EndFor
    \State Construct the GMM surrogate 
    $$
    \pi^K_{M}(y )=\frac{1}{\sqrt{(2 \pi)^d | \Gamma |}} \, \frac{1}{M}\sum_{m=1}^{M} \exp \! \Big( -\frac{1}{2}\| z_m - y \|_\Gamma^2 \Big);
    $$

    \State Set $\vartheta = 0$;
 		\For{$n=1,\dots, N$} 
 		\State Draw an integer $m^*$ from the uniform distribution over $\{1,\ldots, M\}$; 
 			\State Draw $y_n$ from $\mathcal{N}(z_{m^*},\Gamma)$;  \State Set $\vartheta = \vartheta - \log(\pi_M^K(y_n))$;
 		\EndFor\\
   \Return $\frac{\vartheta}{N}$;
 	\end{algorithmic}
 \end{algorithm}
\subsection{Deconvolution}

First, we consider a linear $\calG$ that originates from a weighted (de)convolution problem; cf.,~e.g.,~\cite[Example~9.3]{CalvettiSomersalo2023}. To be precise, we set
\begin{equation*}
   (\calG x)(t)=(g*x) (t)=\int^{1}_{0} g(t-\tau) \, x(\tau)\, w(\tau) \,  \rd\tau, \qquad t \in [0,1],
\end{equation*}
 where $x:[0,1]\to\R$ is the original signal that is to be recovered in the underlying inverse problem, and the Gaussian convolution kernel $g$ and the weight $w$ are defined, respectively, by
 \begin{eqnarray}
     g(t)=\frac{1}{\sqrt{2\pi}\, \gamma}\exp \Big(-\frac{t^2}{2\gamma^2} \Big) \quad \text{and} \quad w(t) = (1-t )^4. \label{eqn:gauss_kernel}
 \end{eqnarray}
 
Consider the grid points $t_k = (k-1)/(K-1)$, $k=1, \dots, K$, and let us introduce the surrogate (or discretized) forward map $\calG_K: \R^{K} \to \R^{K}$ via
\begin{eqnarray*}
    (\calG_K x)_j = \sum_{k=1}^{K} \frac{1}{K-1} g(t_j - t_k) \, (1- t_k)^4 \,  x_k,  \qquad  j=1,\ldots,K,
\end{eqnarray*}
where we have abused the notation by redefining $x$ to be a vector with components $x_k = x(t_k)$, $k=1, \dots, K$. Assuming an additive Gaussian measurement noise $\epsilon$, we arrive at a linear system
 \begin{equation}
 \label{eq:A_linear}
    y=Ax+\epsilon,
\end{equation}
where $y\in \R^{K}$ is the measurement, $x\in \R^{K}$ is the unknown, and we have identified the discretized forward operator with a matrix $A\in \R^{ K\times K}$ given componentwise as
\[
A_{jk}=\frac{1}{K-1}\;g(t_j- t_k) \, (1-t_k)^4, \qquad j, k = 1, \dots, K.
\]

We assume the prior and noise are mutually independent zero-mean Gaussians with diagonal covariance matrices $\Sigma=\sigma_{x}^2 I $ and $\Gamma =\sigma_{\epsilon}^2 I $, respectively. In particular, it follows that the differential entropy of the evidence distribution $\pi^K$ for the model \eqref{eq:A_linear} has the analytic form
\begin{equation}
\label{eq:JK}
     J^K:=\frac{K}{2} \left( 1+\log(2\pi) \right)+\frac{1}{2}\log |A\Sigma A^{\top}+\Gamma|.
\end{equation}
 We aim to estimate $J^K$ in what follows, that is, unlike in Theorem~\ref{thm:MC-GMM-estimator} and \eqref{eq:QMC-Gauss}, we ignore the discrepancy between $\mathcal{G}$ and $\mathcal{G}_K$.  Due to the representation~\eqref{eq:JK}, we can compute the error exactly for each individual realization of our estimators. Moreover, we are only interested in the convergence rate with respect to $M$ since it corresponds to the number of forward operator evaluations. Both randomized QMC and MC are used for constructing the GMM in Algorithm~\ref{alm_gmm}. Note that the considered finite-dimensional model satisfies the conditions of Assumption~\ref{ass:prior_perturbation} for any $L_1$ and $L_2$ that satisfy
\begin{equation}
\label{eq:L1_and_L2}
L_1 \geq \frac{\| A \|_2}{\sigma_\epsilon} \quad \text{and} \quad 0 < L_2 < \frac{1}{2 \sigma_x^2},
\end{equation}
where $\| A \|_2$ is the operator norm with respect to the Euclidean vector norm,~i.e., the largest singular value of $A$.

Our parameter choices are as follows: the dimension of the problem is $d=K=20$, and the prior and noise standard deviations are set to $\sigma_x=10$ and $\sigma_\epsilon=2$, respectively.
The free parameter in the Gaussian kernel (\ref{eqn:gauss_kernel}) is $\gamma=0.1$.  It can be numerically verified that the sufficient condition $L_1^2 < \frac{1}{2} L_2$ {\em cannot} be satisfied with these choices (cf.~\eqref{eq:L1_and_L2}), which means Theorem~\ref{thm:MC-GMM-estimator} does not guarantee the convergence rate in \eqref{eq:theorem2}. 
Moreover, recall that we did not deduce in Section~\ref{subsec:QMC/Gaussian_prior} precise conditions on $\mathcal{G}_K$ that guarantee the convergence rate in \eqref{eq:QMC-Gauss}, even though we also aim to verify that rate numerically in the following (cf.~\eqref{eq:C_sigma}). See Remark~\ref{rem:lin_model} below for further details.

As we are interested in verifying convergence rates in $M$, we choose large enough $N$ to make sure that the $M$-dependent terms dominate in \eqref{eq:theorem2} and \eqref{eq:QMC-Gauss} -- even if the hidden constants associated with the $N$-dependent terms are considerably larger. Figure~\ref{fig:linear2} shows the convergence of the RMSE for the MC and randomized QMC differential entropy estimators with $30$ realizations; as the generating vector $z$ for randomized QMC we employ  \texttt{lattice-32001-}\texttt{1024-1048576.3600} from \cite{FrancesWeb}. To be more precise, for MC the quantity that is plotted with a solid line as a function of $M$ is the right-hand side of the approximate equality
\begin{equation}
\label{eq:RMSE}
 \sqrt{\E^{\otimes \mu} \E^{\otimes \pi_M^K}\big(J^K - \JMN^K\big)^2}
 \approx
 \sqrt{\sum_{p=1}^{30}\frac{1}{30}\big(J^K-\JMN^{K,p}\big)^2} \, ,
\end{equation}
where $\{\JMN^{K,p}\}_{p=1}^{30}$ are independent realizations of the MC estimator $\JMN^K$. For the QMC variant, the outer expectation on the left-hand side is taken over the random shift $\Delta$ in \eqref{eq:rand-lattice-def}, and the independent random realizations of the estimator on the right-hand side are drawn accordingly. Note that via the standard bias-variance decomposition,
\begin{align*}
&\E^\eta \, \E^{\otimes \pi_M^K}\big( J^K- \JMN^K\big)^2
\\&\quad =
\big( J^K - \E^{\eta} \, \E^{\otimes \pi_M^K} \JMN^K\big)^2
+
\E^{\eta} \, \E^{\otimes \pi_M^K} \big (\JMN^K-\E^{\eta} \E^{\otimes \pi_M^K} \JMN^K \big)^2,  
\end{align*}
where $\eta = \otimes \mu$ for the MC estimator and $\eta = \Delta$ for the QMC estimator. The second term on the right-hand side, i.e.~the variance, can be approximated by the sample variance over the different realizations of the estimator. To illustrate the behavior of this quantity, Figure~\ref{fig:linear2} also depicts the square root of the sample variance, i.e.~the standard deviation, as a function of $M$ over the different realizations employed in~\eqref{eq:RMSE}.

\begin{figure}[!h]
\begin{center}
\includegraphics[scale=0.8]{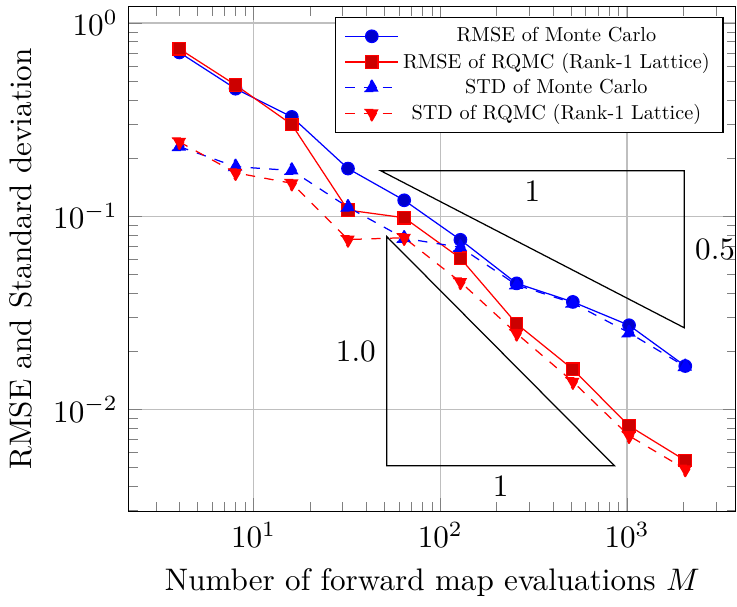}
	\caption{The RMSEs and standard deviations as functions of $M$ for the MC and randomized QMC estimators of the differential entropy $J^K$ given in \eqref{eq:JK} for the linear model \eqref{eq:A_linear}. For both methods, we choose large enough $N$ so that the $M$-dependent terms dominate in \eqref{eq:theorem2} and~\eqref{eq:QMC-Gauss}.
    }\label{fig:linear2}
\end{center}
\end{figure}

Figure~\ref{fig:linear2} verifies the convergence rates for the RMSE with respect to $M$ predicted by \eqref{eq:theorem2} and \eqref{eq:QMC-Gauss}, i.e., $\mathcal{O}(M^{-1/2})$ and $\mathcal{O}(M^{-1 + \gamma})$ for any $\gamma >0$, respectively. Based on numerical experiments not documented here, we also note that the RMSE for the MC-based estimator exhibits a convergence rate closer to ${\mathcal O}(M^{-1})$ for some linear problems with Gaussian prior and noise. This phenomenon may be due to the GMM's ability to provide accurate approximations for Gaussian distributions, which may seem to result in a ``too high'' convergence rate if the studied linear model is simple.

\begin{remark}
\label{rem:lin_model}
The role of the assumption $L_1^2 < \frac{1}{2} L_2$ in Theorem~\ref{thm:MC-GMM-estimator} is to guarantee the finiteness of the expectation in \eqref{eq:bound_for_chi2},~i.e., that
\begin{equation}
\label{eq:lin_cond}
\E^{\pi^K}\E^{\mu} \bigg(\frac{\pi^K(Y \mid X)}{\pi^K(Y)}\bigg)^2 = \int_{\R^d} \int_{\R^K} \frac{\pi^K(y \mid x)^2 \, \pi_{\mu}(x)}{\pi^K(y)} \, {\rm d} x \, {\rm d} y  < \infty,
\end{equation}
where we assume the prior is absolutely continuous with respect to the Lebesgue measure. For a linear measurement model with mutually independent Gaussian (prior) distributions for $X$ and the additive noise, such as is the one for our deconvolution problem, all densities involved in the double integral of \eqref{eq:lin_cond} are Gaussian, which allows explicit evaluation:
\[
\E^{\pi^K}\E^{\mu} \bigg(\frac{\pi^K(Y \mid X)}{\pi^K(Y)}\bigg)^2 = \frac{\det (A \Sigma A^\top + \Gamma)}{\det \Gamma},
\]
which is finite unless the noise covariance $\Gamma$ is degenerate, \emph{even if $L_1^2 < \frac{1}{2} L_2$ does not hold}. This explains why the convergence rate of Theorem~\ref{thm:MC-GMM-estimator} is reached in Figure~\ref{fig:linear2}. It would also be possible to analyze whether the considered linear forward operator satisfies \eqref{eq:C_sigma}, but we omit such considerations to avoid detailed treatment of the space $W^{1,2}_{*}(\R^{K})$. 
\end{remark}

\subsection{Elliptic PDE with random diffusion coefficients}

We next consider a source problem for an elliptic PDE model, where the unknown is a diffusion coefficient and pointwise evaluations of the solution field serve as the measurements. The exactly same model was considered in \cite{kaarnioja2024quasimontecarlobayesiandesign}, and it can,~e.g.,~describe Darcy's flow of fluid within a porous medium.  

To be more precise, we consider the following elliptic PDE problem over the two-dimensional square $D=(0,1)^2$: 
\begin{align}
\label{eqn:pde_model_simple_source_fcn}
\begin{cases}
    -\nabla \cdot ( a(s,x)  \nabla u(s, x)) = 10s_1, & \ \ s\in D, \\[1mm]
    u(s, x)= 0, & \ \ s \in \partial D,
    \end{cases}
\end{align}
where the (weak) derivatives are taken with respect to the spatial variable $s$ and the boundary value is to be understood in the sense of the appropriate Sobolev trace. The diffusion coefficient is defined via a Karhunen--Lo\`eve type expansion,
\begin{align}\label{eqn:pde_parameter_representation}
      a(s,x)  =1+0.1\sum_{j=1}^{K} j^{-2} \big(x_j - \tfrac{1}{2}) \sin(\pi j s_1) \sin(\pi j s_2) ,
\end{align}
where the domain for the unknown parameter $x$ is $[0,1]^K$, with $K = 100$. This can be interpreted as having $\R^{K}$ as the domain for the forward operator accompanied with a uniform prior supported on~$[0,1]^K \subset \R^{K}$. Observe that source and the model for the random diffusion coefficient in \eqref{eqn:pde_parameter_representation} are symmetric with respect to the line $s_2 = 0.5$. The latter can be seen as follows:
\begin{align*}
 a\big((s_1, 1-s_2) ,x\big) &= 1+0.1\sum_{j=1}^{K} j^{-2} \big(x_j - \tfrac{1}{2}) \sin(\pi j s_1) \sin(\pi j (1-s_2)) \\[1mm]
 &= 1 + 0.1\sum_{j=1}^{K} j^{-2} (-1)^{j-1} \big(x_j - \tfrac{1}{2} ) \sin(\pi j s_1) \sin(\pi j s_2),
\end{align*}
where $(-1)^{j-1} (X_j - \tfrac{1}{2} )$ is uniformly distributed over $[-\tfrac{1}{2},\tfrac{1}{2}]$ as is the case for $X_j - \tfrac{1}{2}$ in \eqref{eqn:pde_parameter_representation}.

Because
\[
\sum_{j=1}^{\infty} j^{-2} = \frac{\pi^2}{6},
\]
it is easy to check that 
\[
 0.9 < a(s,x) < 1.1 \quad \text{for all } s \in D, \, x \in [0,1]^{K}. 
\]
As in addition $D$ is a convex polygon and both  $a(\, \cdot \, ,x)$ and the source term are in $C^\infty(\overline{D})$, the problem \eqref{eqn:pde_parameter_representation}  has a unique solution in $u(\, \cdot \,, x ) \in H^2(D)$ for any $x \in [0,1]^{K}$ due to standard theory for elliptic PDEs~\cite{Grisvard1985}. Since $H^2(D) \subset C(\overline{D})$ by virtue of the Sobolev embedding theorem, it is possible to define our measurements as point evaluations of the solution $u(\, \cdot \,, x )$. In fact, $u(\, \cdot \,, x )$ is smooth in the interior of $D$ for any $x \in [0,1]^{K}$ because of interior elliptic regularity. 

We define the nonlinear forward operator as 
\begin{equation}
\label{eq:G_PDE}
\mathcal{G}_K:
\left\{
\begin{array}{l}
\R^{K} \to \R^3, \\[1mm]
x \mapsto \big[u(\varsigma_j, x)\big]_{j=1}^3,
\end{array}
\right.
\end{equation}
where 
the possible values for the components of the measurement points $\varsigma_j \in \R^2$ are $0.25$, $0.5$ and $0.75$, cf.~Figure~\ref{fig:darcy-2d}. We consider three designs out of all $\binom{9}{3}=84$ possibilities (or $47$ possibilities if accounting for the symmetry in the source and the random diffusion coefficient): (i)~the optimal design $ \{\varsigma_1,\varsigma_2,\varsigma_3 \} = \{ (0.50, 0.50), (0.75, 0.50),  (0.50, 0.75) \}$,\footnote{
This optimal design is the same as in \cite{kaarnioja2024quasimontecarlobayesiandesign} after reflection with respect to the line $s_2=0.5$.} (ii)~the worst design $\{(0.25, 0.25),  (0.25, 0.50),$ $  (0.25, 0.75) \}$, and~(iii) the $42$nd best design $\{ (0.25, 0.50), (0.75, 0.25),  (0.75, 0.50) \}$.
The optimal measurement points are visualized in Figure~\ref{fig:darcy-2d} together with the solution of \eqref{eqn:pde_model_simple_source_fcn} for one possible realization of $x$.

Although we do not aim to verify convergence rates in \eqref{eq:theorem2} and \eqref{eq:QMC_rate} for $\mathcal{G}_K$ but only for its discretized version introduced below, let us in any case briefly consider if $\calG_K$ satisfies the assumptions of Theorems~\ref{thm:MC-GMM-estimator} and \ref{thm:QMC-uniform-result}.
As the solution to \eqref{eqn:pde_model_simple_source_fcn} depends analytically on the diffusion coefficient $a(\, \cdot \, , x)$ in the topology of $L^\infty(D)$ (see \cite[Appendix~A]{Garde2020} for a proof in a closely related setting with explicit formulas for Fr\'echet derivatives of all orders) and the dependence of $a(\, \cdot \, , x)$ on $x$ is affine, it can be deduced that $\calG_K \in W^{1,\infty}_{\rm mix}([0,1]^{K})^3$ by resorting to elliptic regularity theory, i.e., $\calG_K$ satisfies the assumptions of  Theorem~\ref{thm:QMC-uniform-result}. Moreover, according to Lemma~\ref{lem:QMC_uniform_lemma}, the condition $\calG_{K} \in W^{1,\infty}_{\rm mix}([0,1]^{K})^3$ is enough to guarantee that Assumption~\ref{ass:prior_perturbation} is satisfied with some $L_1$ and $L_3$ and any $L_2 > 0$, and thus the conditions of Theorem~\ref{thm:MC-GMM-estimator} are also valid.

 The domain $D=(0,1)^2$ is discretized into a regular finite element (FE) mesh with $8192$ right triangles and $4225$ nodes; in particular, the side length for the triangles is $h=1/64$.  For any given $x \in [0,1]^{K}$, a numerical solution to \eqref{eqn:pde_model_simple_source_fcn} is computed by the finite element method with piecewise linear basis functions. The discretized forward operator $\mathcal{G}_{K,h}$ is defined by replacing the solution of \eqref{eqn:pde_model_simple_source_fcn} in \eqref{eq:G_PDE} by its FE approximation.
 Take note that evaluating an FE solution at the measurement points is straightforward as they coincide with certain nodes of the FE mesh. Even though analyzing the discretization error would be possible, we do not stress this matter any further and simply apply Algorithm~\ref{alm_gmm} to approximating the differential entropy of the evidence distribution induced by the 
 $\mathcal{G}_{K,h}$. The studied forward model is 
 \[
 y = \mathcal{G}_{K,h}(x) + \epsilon, 
 \]
where $\epsilon$ is zero-mean Gaussian noise with diagonal covariance  $\Gamma_{\epsilon}=\sigma_{\epsilon}^2 I$, where $\sigma_\epsilon=0.1$. In this problem setting, the analytic form of the entropy is not available, and hence we compute the reference solution using a larger sample size.

\begin{figure}[!ht]
    \centering
\includegraphics[scale=0.6]{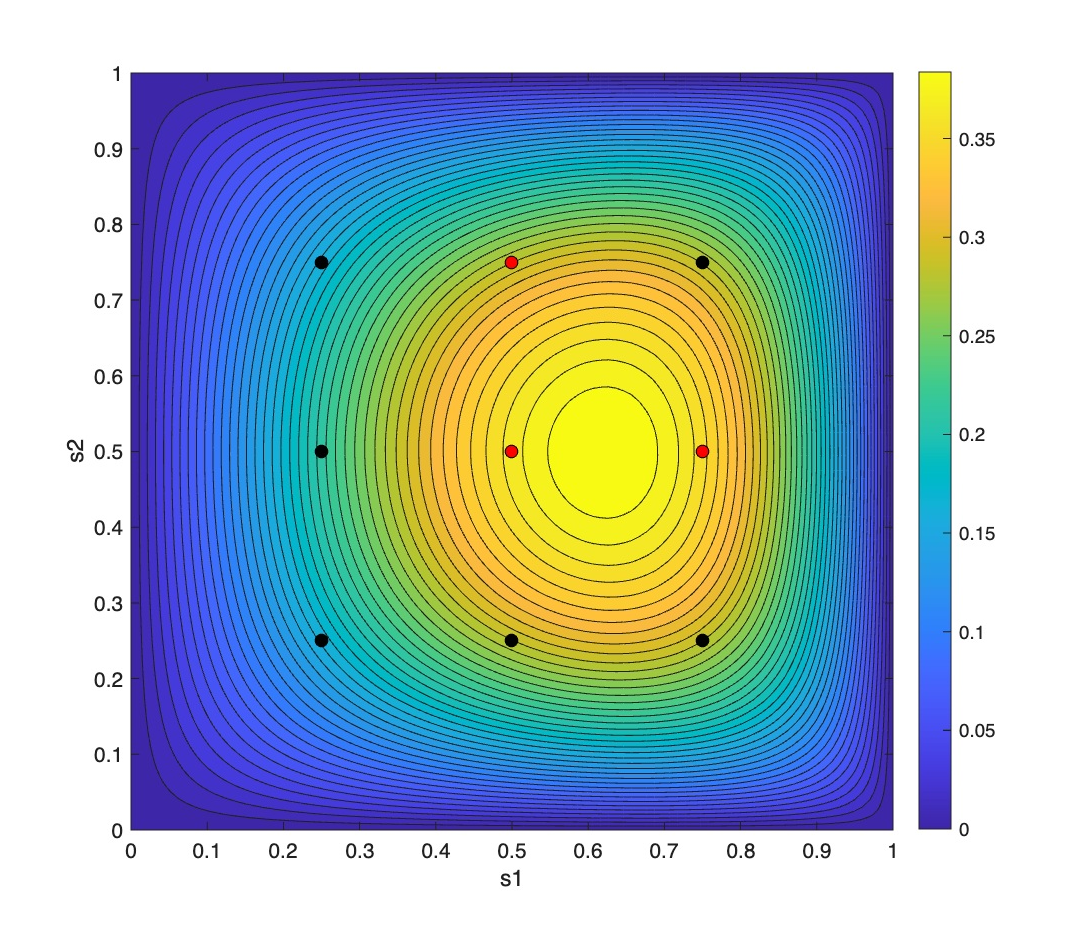} 
\caption{The three optimal observation points (red dots) on top of the solution to \eqref{eqn:pde_model_simple_source_fcn} with one possible realization of $x$. For comparison, the black dots depict the other possible measurement locations, cf. \cite{kaarnioja2024quasimontecarlobayesiandesign} }
    \label{fig:darcy-2d}
\end{figure}

As a deviation from the first numerical experiment, we adopt the idea introduced in Remark~\ref{remark:cubature}: instead of employing the standard MC-based estimator $\JMN^K$ from \eqref{eq:def-estimator}, we estimate the differential entropy of the GMM approximation for the $\mathcal{G}_{K,h}$-induced evidence produced by the first part of Algorithm~\ref{alm_gmm} by resorting to the randomized M\"obius-transformed lattice rule (cf.~\cite{SHK2025,KSG2025}) denoted here by $Q_N^{\widetilde{\Delta}}$. That is, we introduce an alternative estimator
\begin{equation}
\label{eq:Mobius2}
   \widetilde{J}^K_{M,N} = Q_N^{\widetilde{\Delta}} \big(\pi^K_M  \log(\pi^K_M) \big)
   :=\sum_{n=1}^N w_n \pi^K_M (Y_n)  \log(\pi^K_M(Y_n)),
    \end{equation}
where we refrain from explicitly marking the dependence on $h$, the cubature rule based on $\{Y_n,w_n\}_{n=1}^N$ replaces the second loop in Algorithm~\ref{alm_gmm} and $\widetilde{\Delta}$ refers to the random shift in the underlying lattice rule (cf.~\eqref{eq:rand-lattice-def}). The reason for this modification is that the presumed higher convergence rate of the randomized M\"obius-transformed lattice rule enables seeing the predicted convergence rate in $M$ with fewer evaluations of the GMM approximation $\pi^K_M$ for the target evidence density $\pi^K$. Indeed, this is achieved by choosing $N = 2^{20}$ in all evaluations of the estimator $\widetilde{J}^K_{M,N}$ in the numerical tests. As an additional alteration compared to the first experiment, we test the idea in Remark~\ref{rem:ho-qmc} and also consider a higher-order QMC method, i.e., the tent-transformed shifted lattice rule \cite{H2000,GSY2019} in the first part of Algorithm~\ref{alm_gmm}.

For the choice of the generating vector $z$ of the randomly shifted rank-1 lattice points in \eqref{eq:rand-lattice-def}, we use off-the-shelf lattice sequences generated by the CBC construction \cite{CKN2006,NC2006}: (i) for constructing the GMM, we use \texttt{exod2\_base2\_m13.txt} from \cite{DirkWeb}, that is, the first $100$ components given in the above reference; and (ii) for computing the differential entropy of $\pi^K_M$ using M\"obius-transformed lattice points, we again employ \texttt{lattice-32001-1024-1048576.3600} from~\cite{FrancesWeb}, i.e., $z=(1,182667,469891)$. The reason for these choices is to avoid using two identical lattices for two different approximation steps. Both generating vectors are freely accessible online from the referenced websites, and we omit writing the first generating vector (i) consisting of $100$ values.

As there is no analytic representation for the target differential entropy, we analyze the convergence of the estimator $\widetilde{J}^K_{M,N}$ in comparison to a reference value $\widetilde{J}^K_{\rm ref} = \widetilde{J}^K_{M_0,N_0}$ that is computed with the randomized tent-transformed QMC lattice rule with $M_0=2^{13}$ and $N_0=2^{20}$. Figure~\ref{fig:pde_10s1_higher_order_qmc} shows the convergence of the RMSE when using MC and the two randomized QMC rules, with 30 realizations, for building the QMC surrogate in the first loop of Algorithm~\ref{alm_gmm}. More precisely, for MC the quantity plotted as a function of $M$ is the right-hand side of the approximate equality
\[
 \sqrt{\E^{\otimes \mu} \E^{\widetilde{\Delta}} \big( \widetilde{J}^K_{\rm ref} - \widetilde{J}^K_{M,N} \big)^2}
 \approx
 \sqrt{\sum_{p=1}^{30}\frac{1}{30}\big(\widetilde{J}^K_{\rm ref} - \widetilde{J}^{K,p}_{M,N} \big)^2},
\]
where $\{\widetilde{J}^{K,p}_{M,N} \}_{p=1}^{30}$ are independent realizations of the estimator $\widetilde{J}^K_{M,N}$, with a ``realization'' also including drawing a random shift for the Möbius-transformed lattice rule in \eqref{eq:Mobius2}. For the QMC variants, the first expectation on the left-hand side is taken over the random shift in the employed randomized QMC rule for building the GMM, and the 30 independent random realizations of the estimator on the right-hand side are drawn accordingly.

\begin{figure}[!h]
\begin{center}
\includegraphics[scale=1.1]{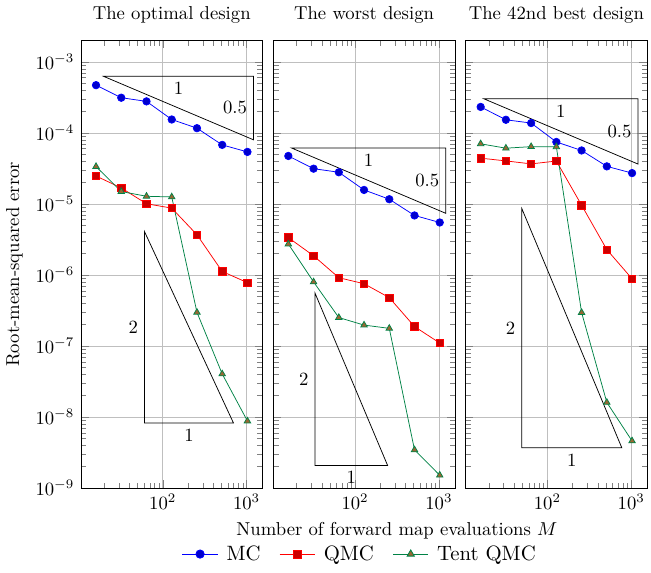}
	\caption{The RMSEs as functions of $M$ for the MC and the two randomized QMC estimators in comparison to the reference differential entropy $\widetilde{J}^K_{\rm ref}$ for the evidence of the nonlinear model \eqref{eq:G_PDE}. The employed QMC methods in the first part of Algorithm~\ref{alm_gmm} are the randomized rank-1 lattice rule (first order method) and the randomized tent-transformed lattice rule (second order method). For all methods, we choose $N=2^{20}$, which suffices for the $M$-dependent terms to dominate in the estimation error (cf. \eqref{eq:theorem2} and~\eqref{eq:QMC_rate}). We consider three possible designs; the optimal design (left): $\{(0.50, 0.50), (0.75, 0.50),  (0.50, 0.75)\}$, the worst design (middle):  $\{(0.25, 0.25), (0.25, 0.50),  (0.25, 0.75)\}$, and the 42nd best design (right): $\{(0.25, 0.50), (0.75, 0.25),  (0.75, 0.50)\}$. For the reference values, needed for computing the RMSEs, we used the tent-transformed lattice rule with $M_0=2^{13}$ and the same generating vector $z$ from \cite{DirkWeb}.
    } \label{fig:pde_10s1_higher_order_qmc}
\end{center}
\end{figure}

When interpreting the convergence rates in Figure~\ref{fig:pde_10s1_higher_order_qmc}, one should note that in~\eqref{eq:theorem2} and~\eqref{eq:QMC_rate}, the convergence rate in $N$ is, in essence, dictated by the method for estimating the differential entropy for the GMM surrogate in the second part of Algorithm~\ref{alm_gmm} --- the motivation for employing the M\"obius-transformed lattice rule with large enough $N$ for this step is to make the $N$-dependent term negligible compared to the $M$-dependent term. On the other hand, the convergence rate with respect to $M$ in \eqref{eq:theorem2} and~\eqref{eq:QMC_rate} is determined by the method used for building the GMM in the first part of Algorithm~\ref{alm_gmm}. This means that one would hope to observe the rate $\mathcal{O}(M^{-1/2})$ for the MC-based GMM, approximately the rate $\mathcal{O}(M^{-1})$ for the first order QMC-based GMM (randomized rank-1 lattice), and approximately the rate $\mathcal{O}(M^{-2})$ for the second order QMC-based GMM (the randomized tent-transformed lattice rule), cf.~Remark~\ref{rem:ho-qmc}. Although these conclusions are only heuristic extrapolations of Theorems~\ref{thm:MC-GMM-estimator} and \ref{thm:QMC-uniform-result} and Remark~\ref{rem:ho-qmc} since our theoretical results do not cover the M\"obius-transformed lattice rule for computing the differential entropy of a GMM surrogate, 
the convergence rates in Figure~\ref{fig:pde_10s1_higher_order_qmc} anyway seem to be approximately of the anticipated orders.

To support the findings in Figure~\ref{fig:pde_10s1_higher_order_qmc}, we complete this section by repeating the numerical test on the optimal design $\{\varsigma_1,\varsigma_2,\varsigma_3 \} = \{ (0.50, 0.50), (0.75, 0.50), \;\;$ $ (0.50, 0.75) \}$ for three different combinations of the truncation index $K$ 
in \eqref{eqn:pde_parameter_representation} 
and the mesh parameter $h$, namely $(K,h) = (50,1/64)$, $(50,1/128)$ and $(100,1/128)$. The results are illustrated in Figure~\ref{fig:pde_10s1_higher_order_qmc2}, with convergence rates that are in line with those in Figure~\ref{fig:pde_10s1_higher_order_qmc}.

\begin{figure}[!h]
\begin{center}
\includegraphics[scale=1.1]{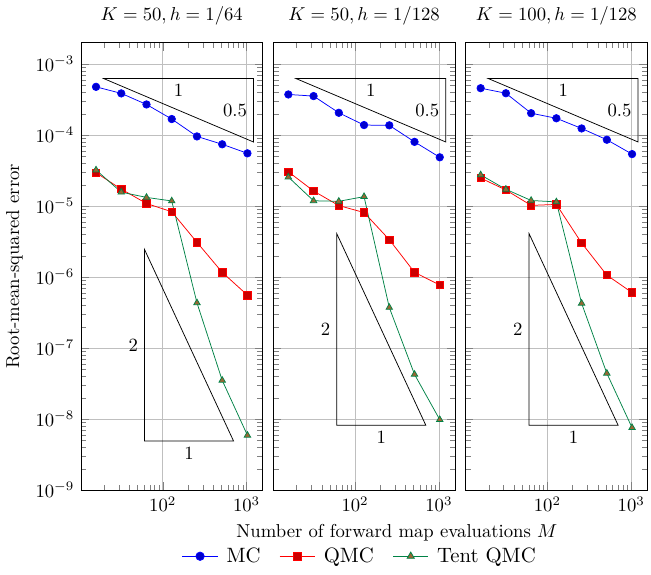}
	\caption{
    The RMSEs as functions of $M$ for the MC and the two randomized QMC estimators in comparison to the reference differential entropy $\widetilde{J}^K_{\rm ref}$ for the evidence of the nonlinear model \eqref{eq:G_PDE}. The employed QMC methods in the first part of Algorithm~\ref{alm_gmm} are the randomized rank-1 lattice rule (first order method) and the randomized tent-transformed lattice rule (second order method). For all methods, we choose $N=2^{20}$, which suffices for the $M$-dependent terms to dominate in the estimation error (cf. \eqref{eq:theorem2} and~\eqref{eq:QMC_rate}). We consider the optimal design $\{(0.50, 0.50), (0.75, 0.50),  (0.50, 0.75)\}$ with three different combinations of the truncation index $K$ in \eqref{eqn:pde_parameter_representation} and the mesh parameter $h$: $(K,h) = (50,1/64)$, $(50,1/128)$ and $(100,1/128)$. For the reference values, needed for computing the RMSEs, we used the tent-transformed lattice rule with $M_0=2^{13}$ and the same generating vector $z$ from \cite{DirkWeb}.
    }
     \label{fig:pde_10s1_higher_order_qmc2}
\end{center}
\end{figure}

\begin{remark}
This example from \cite{kaarnioja2024quasimontecarlobayesiandesign} satisfies the conditions required in their analysis. The inner integrand considered
in \cite[Section~6]{kaarnioja2024quasimontecarlobayesiandesign} is the
unnormalized posterior density and thus, for the uniform prior considered here, it coincides with our posterior up to a normalizing constant (see \eqref{eq:QMCquadrature}--\eqref{eq:uni_quad}). Hence, the POD-weighted lattice construction of \cite[Theorem~6.1]{kaarnioja2024quasimontecarlobayesiandesign} can be directly applied to our evidence cubature, yielding the rate $\calO(M^{-1+\gamma})$, $\gamma > 0$, with a dimension-independent constant. This suggests that the dimension dependence of the hidden constant in Theorem~\ref{thm:QMC-uniform-result} is a consequence of our weaker, unweighted regularity framework rather than a limitation of the proposed GMM construction. Further, under corresponding weighted second-order regularity and summability assumptions, tent-transformed lattice rules may additionally attain the rate $\calO(M^{-2+ \gamma})$ with a dimension-independent constant, cf.~\cite[Theorem~2]{GSY2019}.
    \label{rem:qmc-dependence}
\end{remark}

\section{Conclusion}\label{sec:conclusion}
We introduced an efficient method for approximating the differential entropy of the evidence distribution for a class of inverse problems. The algorithm can be employed in evaluating the expected information gain, the maximization of which is commonly considered in Bayesian OED. Our focus was on reducing the total number of forward map evaluations which was assumed to dominate the computational cost in the considered problem settings. By constructing a surrogate for the evidence $\pi( \, \cdot \, ;\xi)$ via GMM, given a design $\xi$, our method avoids directly computing the nested integral often encountered in Bayesian OED and separates the original problem into two different approximation steps for the unknown and the data. The convergence rate of the MC variant of the proposed method is faster than for standard methods (if measured by the number of forward map evaluations), and this rate can be further accelerated by resorting to QMC techniques. The numerical experiments supported our theoretical findings on the convergence rate in the number of forward operator evaluations. 

\section*{Acknowledgement}
We thank Antti Hannukainen for letting us use his FE codes in our numerical experiments. This work was supported by the Research Council of Finland (decisions 348503, 348504, 359181, 359183). A part of the numerical experiments was performed using computer resources provided by the Aalto Science-IT project and the cluster service in LUT University.

\appendix

\section{On derivatives of the likelihood}
In this appendix, we follow the convention of Section~\ref{sec:QMC_uniform} and consider the behavior of the $x$-derivatives of the unnormalized likelihood
\[
\mathcal{L}(x,y) := \exp\!\big( -\Phi_K(x,y) \big) = \exp\!\Big( -\frac{1}{2} \| \mathcal{F}_K(x,y) \|_{\Gamma}^2 \Big), \qquad x \in [0,1]^K, y \in \R^d 
\]
under an additive Gaussian noise model. We have adopted the shorthand notation
\[
\mathcal{F}_K(x,y) = \mathcal{G}_K(x) - y 
\]
to keep the statement of the following proposition as concise as possible.

\begin{proposition}
\label{prop:alpha_deriv}
Let $0 \not= \tau \in \N_0^K$ and assume that
\[
\frac{\partial^{|\mu|}}{\partial x^{\mu}} \mathcal{G}_K \in L^\infty([0,1]^K)^d
\]
for any $\mu \in \N_0^K$ such that $\mu \leq \tau$. Then, the (weak) derivative
\begin{equation}
\label{eq:alpha_deriv}
\frac{\partial^{|\tau|}}{\partial x^{\tau}} \mathcal{L}(x,y) 
\end{equation}
is a finite linear combination of terms of the form 
\begin{equation}
\label{eq:horror_deriv}
\mathcal{L}(x,y) \prod_{j=1}^J \frac{\partial^{|\beta^{j}|}}{\partial x^{\beta^{j}}} \mathcal{F}_K(x,y)^\top  \Gamma^{-1}  \frac{\partial^{|\gamma^{j}|}}{\partial x^{\gamma^{j}}} \mathcal{F}_K(x,y),
\end{equation}
where $J \in \N$ and the multi-indices $\beta^{j}, \gamma^{j} \in \N_0^K$ satisfy
\begin{equation}
\label{eq:horror_ind}
J \leq |\tau|, \qquad |\beta^{j}| + |\gamma^{j}| \not= 0, \qquad \sum_{j=1}^J \big( |\beta^{j}| + |\gamma^{j}| \big) = | \tau |, \qquad \beta^{j}_k, \gamma^{j}_k \leq  \tau_k 
\end{equation}
for all $j=1, \dots, J$ and $k=1, \dots, K$. In particular, $|\beta^{j}|_\infty, |\gamma^{j}|_\infty \leq |\tau|_\infty$. 
\end{proposition}

\begin{proof}
To begin with, note that the assumption on the regularity of the derivatives of $\mathcal{G}_K$ ensures that the product rule can be freely employed to calculate the derivatives in~\eqref{eq:alpha_deriv}. That is, all functions that are differentiated belong to $W^{1,\infty}([0,1]^K)$.

The actual proof is based on induction with respect to $|\tau| \in \N$. Let $|\tau| = 1$, which means that $\tau = e_l$ for some $l \in \{1, \dots, K\}$, with $e_l$ denoting the unit multi-index in the direction of the $l$th component. In consequence,
\[
\frac{\partial^{|\tau|}}{\partial x^{\tau}} \mathcal{L}(x,y) = \frac{\partial}{\partial x_l} \mathcal{L}(x,y) = - \mathcal{L}(x,y) \mathcal{F}_K(x,y)^{\top} \Gamma^{-1}  \frac{\partial}{\partial x_l} \mathcal{F}_K(x,y), \qquad l \in \{1, \dots, K\},
\]
by symmetry of $\Gamma^{-1}$. This verifies the claim for $|\tau| = 1$.

Assume then that the claim holds for any such $\tau \in \N_0^K$ that $|\tau| = n \in \N$.  To take the induction step, let $\tau$, with $|\tau| = n$, be arbitrary and consider $\tilde{\tau} = \tau + e_l$ for an arbitrary but fixed $l \in \{1, \dots, K\}$. For such $\tilde{\tau}$, it holds that $|\tilde{\tau}| = n+1$, $\tilde{\tau}_l = {\tau}_l + 1$, and $\tilde{\tau}_k = {\tau}_k$ for $k\not= l$. Due to the induction assumption, the derivative 
\[
\frac{\partial^{|\tilde{\tau}|}}{\partial x^{\tilde{\tau}}} \mathcal{L}(x,y) 
\]
is a finite linear combination of terms of the form (with $J$, $\beta^{j}$, and $\gamma^{j}$ satisfying \eqref{eq:horror_ind})
\begin{align*}
\frac{\partial}{\partial x_l} & \Bigg(
\mathcal{L}(x,y) \bigg( \prod_{j=1}^J \frac{\partial^{|\beta^{j}|}}{\partial x^{\beta^{j}}} \mathcal{F}_K(x,y)^\top  \Gamma^{-1}  \frac{\partial^{|\gamma^{j}|}}{\partial x^{\gamma^{j}}} \mathcal{F}_K(x,y) \bigg) \Bigg) \\
&= \mathcal{L}(x,y)  \bigg( -\Big( \mathcal{F}_K(x,y)^\top \Gamma^{-1}  \frac{\partial}{\partial x_l} \mathcal{F}_K(x,y) \Big) \prod_{j=1}^J \frac{\partial^{|\beta^{j}|}}{\partial x^{\beta^{j}}} \mathcal{F}_K(x,y)^\top  \Gamma^{-1}  \frac{\partial^{|\gamma^{j}|}}{\partial x^{\gamma^{j}}} \mathcal{F}_K(x,y) \\
& \quad + \sum_{k=1}^J  \Big(\frac{\partial^{|\beta^{k}|+1}}{\partial x^{\beta^{k}+e_l}} \mathcal{F}_K(x,y)^\top  \Gamma^{-1}  \frac{\partial^{|\gamma^{k}|}}{\partial x^{\gamma^{k}}} \mathcal{F}_K(x,y) \Big) \prod_{j \not= k} \frac{\partial^{|\beta^{j}|}}{\partial x^{\beta^{j}}} \mathcal{F}_K(x,y)^\top  \Gamma^{-1}  \frac{\partial^{|\gamma^{j}|}}{\partial x^{\gamma^{j}}} \mathcal{F}_K(x,y)  \\
&\quad + \sum_{k=1}^J  \Big(\frac{\partial^{|\beta^{k}|}}{\partial x^{\beta^{k}}} \mathcal{F}_K(x,y)^\top  \Gamma^{-1}  \frac{\partial^{|\gamma^{k}|+1}}{\partial x^{\gamma^{k}+e_l}} \mathcal{F}_K(x,y)\Big)  \prod_{j \not= k} \frac{\partial^{|\beta^{j}|}}{\partial x^{\beta^{j}}} \mathcal{F}_K(x,y)^\top  \Gamma^{-1}  \frac{\partial^{|\gamma^{j}|}}{\partial x^{\gamma^{j}}} \mathcal{F}_K(x,y) \bigg),
\end{align*}
where we employed the product rule and the symmetry of $\Gamma^{-1}$. It is straightforward to check that each term in the linear combination on the right-hand side is of the form \eqref{eq:horror_deriv} for $\tilde{J} \in \N$ and $\tilde{\beta}^{j}, \tilde{\gamma}^{j} \in \N_0^K$, $j=1, \dots, \tilde{J}$, satisfying \eqref{eq:horror_ind} with $\tau$ replaced by $\tilde{\tau}$. As $l \in \{1, \dots, K\}$ was chosen arbitrarily, this completes the proof. 
\end{proof}

\begin{corollary}
\label{eq:alpha_deriv_bound}
Let $\alpha\in \N_0$ and assume $\mathcal{G}_K \in W^{\alpha,\infty}_{{\rm mix}}([0,1]^{K})^d$. Then, for each $\tau \in \N_0^K$ such that $| \tau |_\infty \leq \alpha$,
\[
\Big| \frac{\partial^{|\tau|}}{\partial x^{\tau}} \mathcal{L}(x,y) \Big| \leq \mathcal{L}(x,y) \sqrt{q_\tau (y)}, \qquad x \in (0,1)^K, \ y \in \R^d, 
\]
where $q_\tau: \R^d \to \R$ is a polynomial of degree $2|\tau|$ whose coefficients depend on $\tau$, $\| \Gamma^{-1}\|_2$, $d$ and $\| \mathcal{G}_K \|_{W^{\alpha,\infty}_{{\rm mix}}([0,1]^{K})^d}$.
    \end{corollary}

\begin{proof}
To begin with, note that the case $\tau = 0$ holds trivially and can thus be excluded from what follows.

By virtue of Proposition~\ref{prop:alpha_deriv} and the triangle inequality, it suffices to prove there exists a polynomial $q$ of degree $2 |\tau|$ such that
\begin{align}
\label{eq:alpha_norm1}
\ell(x,y) &:= \bigg |\prod_{j=1}^J \frac{\partial^{|\beta^{j}|}}{\partial x^{\beta^{j}}} \mathcal{F}_K(x,y)^\top  \Gamma^{-1}  \frac{\partial^{|\gamma^{j}|}}{\partial x^{\gamma^{j}}} \mathcal{F}_K(x,y) \bigg| \nonumber \\
& \ \leq \| \Gamma^{-1} \|_2^J \prod_{j=1}^J  \bigg \| \frac{\partial^{|\beta^{j}|}}{\partial x^{\beta^{j}}} \mathcal{F}_K(x,y) \bigg\|_2 \bigg \| \frac{\partial^{|\gamma^{j}|}}{\partial x^{\gamma^{j}}} \mathcal{F}_K(x,y) \bigg\|_2 \leq \sqrt{q(y)}
\end{align}
for all $J \leq |\tau|$ and $\beta^{j}, \gamma^{j} \in \N_0^K$ such that $|\beta^{j}|_\infty, |\gamma^{j}|_\infty \leq \alpha$ and $|\beta^j| + |\gamma^j| \not= 0$. Without loss of generality, we may assume $\gamma^j \not=0$ for all $j=1, \dots, J$.  

Let us consider an arbitrary term in the product on the right-hand side of \eqref{eq:alpha_norm1}:
\begin{align*}
 \bigg \| \frac{\partial^{|\beta^{j}|}}{\partial x^{\beta^{j}}} \mathcal{F}_K(x,y) \bigg\|_2 & \bigg \| \frac{\partial^{|\gamma^{j}|}}{\partial x^{\gamma^{j}}} \mathcal{F}_K(x,y) \bigg\|_2 \leq \bigg(\bigg \| \frac{\partial^{|\beta^{j}|}}{\partial x^{\beta^{j}}} \mathcal{G}_K(x) \bigg\|_2 + \bigg\| \frac{\partial^{|\beta^{j}|}y}{\partial x^{\beta^{j}}}  \bigg\|_2 \bigg) \bigg\| \frac{\partial^{|\gamma^{j}|}}{\partial x^{\gamma^{j}}} \mathcal{G}_K(x) \bigg\|_2 \\[2mm]
 &\leq
\sqrt{d} \, \| \mathcal{G}_K \|_{W^{\alpha,\infty}_{\text{mix}}([0,1]^{K})^d} \Big(\sqrt{d} \, \| \mathcal{G}_K \|_{W^{\alpha,\infty}_{\text{mix}}([0,1]^{K})^d} + \| y\|_2 \Big),   
\end{align*}
which is bounded by the square root of the second degree polynomial
\[
p(y) =   2 d \, \| \mathcal{G}_K \|^2_{W^{\alpha,\infty}_{\text{mix}}([0,1]^{K})^d} \Big( d \,  \| \mathcal{G}_K \|_{W^{\alpha,\infty}_{\text{mix}}([0,1]^{K})^d}^2 + \| y \|^2_2 \Big).
\]
Combined with \eqref{eq:alpha_norm1}, this leads to
\begin{align*}
\ell(x,y) \leq   \| \Gamma^{-1} \|_2^J \, \sqrt{p(y)^J},
\end{align*}
which completes the proof since $J \leq |\tau|$.
\end{proof}

\bibliographystyle{siam}
\bibliography{ref.bib}

\end{document}